\documentclass[11pt]{article}
\usepackage[papersize={8.5in,11in},top=1in,left=1in,right=1in,bottom=1in]{geometry}

\usepackage{wrapfig}

\usepackage{todonotes}

\usepackage{microtype} 

\usepackage[utf8]{inputenc} 
\usepackage[T1]{fontenc}    
\usepackage{hyperref}       
\usepackage{url}            
\usepackage{booktabs}       
\usepackage{amsfonts}       
\usepackage{nicefrac}       

\usepackage{xcolor}

\usepackage{algorithm}
\usepackage{algpseudocode}
\usepackage{amsmath}
\allowdisplaybreaks[3] 

\usepackage{amssymb}
\usepackage{amsthm}
\usepackage{mathtools}
\usepackage{braket}
\usepackage[utf8]{inputenc}
\usepackage{colortbl}
\usepackage{enumitem}
\usepackage{lipsum}
\usepackage{wrapfig}
\usepackage[square,numbers]{natbib}
\usepackage{physics}
\usepackage{xparse}

\usepackage{tabularx} 
\usepackage{booktabs} 
\usepackage{ragged2e} 

\newcolumntype{Y}{>{\Centering\arraybackslash}X}

\algnewcommand{\LineComment}[1]{\State $\triangleright$ #1}

\DeclareMathOperator*{\R}{\mathbb{R}}

\DeclareMathOperator{\E}{\mathbb{E}}
\DeclareMathOperator*{\cO}{\mathcal{O}}

\DeclareMathOperator*{\cS}{\mathcal{S}}
\DeclareMathOperator*{\cA}{\mathcal{A}}

\DeclareMathOperator{\KL}{\mathrm{KL}}
\DeclareMathOperator{\J}{\mathbf{J}}

\NewDocumentCommand{\pd}{o}{%
  \mathord{\partial}\IfValueT{#1}{_{#1}}\mkern-0.5mu%
}

\usepackage{authblk}

\newtheorem{assumption}{Assumption}

\newtheorem{theorem}{Theorem}

\newtheorem{lemma}{Lemma}

\title{Breaking the Bias Barrier in Concave Multi-Objective Reinforcement Learning}

\author{Swetha Ganesh and Vaneet Aggarwal}

\affil{Purdue University, USA 47907}
\date{\vspace{-.1in}}

\begin{document}

\maketitle

\begin{abstract}
While standard reinforcement learning optimizes a single reward signal, many applications require optimizing a nonlinear utility $f(J_1^\pi,\dots,J_M^\pi)$ over multiple objectives, where each $J_m^\pi$ denotes the expected discounted return of a distinct reward function. A common approach is concave scalarization, which captures important trade-offs such as fairness and risk sensitivity. However, nonlinear scalarization introduces a fundamental challenge for policy gradient methods: the gradient depends on $\partial f(J^\pi)$, while in practice only empirical return estimates $\hat J$ are available. Because $f$ is nonlinear, the plug-in estimator is biased ($\mathbb{E}[\partial f(\hat J)] \neq \partial f(\mathbb{E}[\hat J])$), leading to persistent gradient bias that degrades sample complexity. 

In this work we identify and overcome this bias barrier in concave-scalarized multi-objective reinforcement learning. We show that existing policy-gradient methods suffer an intrinsic $\widetilde{\mathcal{O}}(\epsilon^{-4})$ sample complexity due to this bias. To address this issue, we develop a Natural Policy Gradient (NPG) algorithm equipped with a multi-level Monte Carlo (MLMC) estimator that controls the bias of the scalarization gradient while maintaining low sampling cost. We prove that this approach achieves the optimal $\widetilde{\mathcal{O}}(\epsilon^{-2})$ sample complexity for computing an $\epsilon$-optimal policy. Furthermore, we show that when the scalarization function is second-order smooth, the first-order bias cancels automatically, allowing vanilla NPG to achieve the same $\widetilde{\mathcal{O}}(\epsilon^{-2})$ rate without MLMC. Our results provide the first optimal sample complexity guarantees for concave multi-objective reinforcement learning under policy-gradient methods.

\end{abstract}

\section{Introduction}

Standard reinforcement learning (RL) studies how to learn a policy $\pi$ that maximizes a single scalar objective $J^\pi$, defined as the expected discounted return of a reward signal. However, many modern decision-making systems must balance multiple competing objectives rather than optimize a single reward. Examples include trading off throughput and energy consumption in communication systems \cite{Aggarwal2017JointEA}, latency and power usage in queueing and computing systems \cite{10.1109/TNET.2020.2973224}, and efficiency-safety trade-offs in robotic control \cite{10.1109/IROS45743.2020.9341519}. In such settings, the performance of a policy is naturally described by a vector of returns $\{J_m^\pi\}_{m=1}^M$, or equivalently $\mathbf{J}^\pi \in \mathbb{R}^M$.

A widely used approach for handling multiple objectives is to introduce a scalarization function and optimize the utility $
f(\mathbf{J}^\pi)$, where $f:\mathbb{R}^M \to \mathbb{R}$ is a concave function that captures the trade-offs among objectives. Standard RL is recovered as the special case $M=1$ with $f(x)=x$. Concave scalarizations are particularly appealing because they naturally encode preferences such as fairness or risk sensitivity. For example, $\alpha$-fair utilities widely used in networking, queueing systems, and multi-task learning take the form
\begin{align}\label{eq:alph-fairness}
f(\mathbf{J}^\pi) = \sum_{m=1}^M u_\alpha(J_m^\pi), 
\qquad 
u_\alpha(x)=\frac{x^{1-\alpha}}{1-\alpha},
\end{align}
where $\alpha$ determines the desired fairness-efficiency trade-off \cite{fairness_network,ban2024fair}.

Despite its conceptual simplicity, optimizing nonlinear utilities in reinforcement learning poses significant theoretical challenges. In particular, value-based methods rely heavily on the additive structure of scalar rewards and Bellman equations, which breaks down under nonlinear scalarization \cite{agarwal2022multi,agarwal2023reinforcement}. As a result, policy-gradient methods provide a natural alternative. However, obtaining sharp sample-complexity guarantees for policy-gradient algorithms in this setting has proven to be considerably more difficult \cite{10.1613/jair.1.13981,barakat2025on,NEURIPS2024_010c855d}.

The main challenge stems from estimating the gradient of the scalarized objective. The policy gradient depends on $\partial f(\mathbf{J}^\pi)$ evaluated at the true return vector. In practice, however, $\mathbf{J}^\pi$ is unknown and must be estimated from sampled trajectories. Consequently, policy-gradient estimators rely on plug-in estimates $\widehat{\mathbf{J}}$ and evaluate $\partial f(\widehat{\mathbf{J}})$. Because the scalarization $f$ is nonlinear, this introduces a fundamental bias: even when $\widehat{\mathbf{J}}$ is an unbiased estimate of $\mathbf{J}^\pi$, the gradient estimator satisfies
\[
\mathbb{E}[\partial f(\widehat{\mathbf{J}})] \neq \partial f(\mathbb{E}[\widehat{\mathbf{J}}]).
\]
This bias persists across iterations and must be controlled through large batch sizes, which substantially increases the sample complexity of policy-gradient methods.

This phenomenon explains the gap in existing theoretical guarantees. The first model-based analysis of concave multi-objective RL was developed in \cite{agarwal2022multi,agarwal2023reinforcement}, which study the average-reward setting and establish regret guarantees using explicit model estimation. Further, \cite{10.1613/jair.1.13981} proposed a model-free policy-gradient algorithm and proved that computing an $\epsilon$-optimal policy requires $\widetilde{\cO}(\epsilon^{-4})$ samples. This rate is significantly worse than the optimal $\widetilde{\cO}(\epsilon^{-2})$ sample complexity known for standard reinforcement learning \cite{mondal2024improved,fatkhullin2023stochastic}. The degradation arises precisely from the bias introduced by nonlinear scalarization.

These observations motivate the central question of this work:
\[
\fbox{\begin{minipage}{0.9\linewidth}\centering
Can we overcome the bias introduced by nonlinear scalarization and obtain an $\epsilon$-optimal policy for concave multi-objective reinforcement learning using only $\widetilde{\cO}(\epsilon^{-2})$ samples?
\end{minipage}}
\]

In this paper we answer this question in the affirmative. Our approach combines natural policy gradient (NPG) methods with bias-controlled gradient estimators that mitigate the effects of nonlinear scalarization.

\paragraph{Technical overview and novelty.}
Our analysis reveals that nonlinear scalarization fundamentally alters the bias-variance trade-off in policy gradient estimation. We show that under standard Lipschitz assumptions the plug-in gradient estimator incurs a bias of order $\mathcal{O}(1/\sqrt{B})$ for batch size $B$, which forces existing policy-gradient methods to use $\cO(\epsilon^{-2})$ samples per iteration and leads to the overall $\widetilde{\cO}(\epsilon^{-4})$ complexity.

Our key idea is to explicitly control this bias. We develop a multi-level Monte Carlo (MLMC) estimator that effectively simulates large-batch gradient estimates while requiring only logarithmic expected sampling cost. When combined with natural policy gradient updates, this estimator reduces the scalarization bias sufficiently to recover the order-optimal $\widetilde{\cO}(\epsilon^{-2})$ sample complexity. Furthermore, we show that when the scalarization function is twice differentiable, the leading-order bias cancels automatically through a second-order expansion, enabling vanilla NPG to achieve the same $\widetilde{\cO}(\epsilon^{-2})$ rate without MLMC. These results establish the first optimal sample complexity guarantees for policy-gradient methods in concave multi-objective reinforcement learning.

\subsection{Main Contributions}

Our main contributions are as follows.

\begin{itemize}
\item \textbf{Optimal sample complexity via MLMC-NPG.}
We develop a natural policy gradient algorithm equipped with a multi-level Monte Carlo (MLMC) estimator that controls the bias arising from nonlinear scalarization. We show that this method achieves $\widetilde{\cO}(\epsilon^{-2})$ sample complexity for computing an $\epsilon$-optimal policy, matching the optimal rate known for standard reinforcement learning.

\item \textbf{Bias cancellation under second-order smooth scalarizations.}
We further show that when the scalarization function $f$ satisfies a second-order smoothness condition, the leading-order bias in the plug-in estimator cancels automatically. In this regime, even vanilla NPG achieves $\widetilde{\cO}(\epsilon^{-2})$ sample complexity without requiring MLMC.
\end{itemize}

\subsection{Related Works}

\paragraph{Concave multi-objective reinforcement learning.}
The formulation of reinforcement learning with concave scalarized utilities was first studied in the model-based setting by \cite{agarwal2022multi,agarwal2023reinforcement}, which consider the average-reward setting and establish regret guarantees using explicit model estimation. Further, \cite{agarwal2022concave} extends the results with constraints. While these approaches provide strong theoretical guarantees, they rely on estimating the underlying MDP model, which can be computationally expensive in large state-action spaces. A related formulation of scalarized multi-objective RL has been studied in \cite{10.5555/3709347.3743798}. However, this paper provides an $\epsilon$-approximation guarantee with runtime complexity for a model-based planner, rather than a sample complexity guarantee. 

\paragraph{Policy gradient methods for concave utilities.}
To address the scalability limitations of model-based methods, \cite{10.1613/jair.1.13981} proposed a model-free policy-gradient algorithm for concave multi-objective reinforcement learning and showed that computing an $\epsilon$-optimal policy requires $\cO(\epsilon^{-4})$ samples. Subsequent work \cite{zhou2022anchorchanging} analyzes natural policy-gradient methods for concave utility reinforcement learning, but assumes access to exact gradients, thereby bypassing the sampling challenges that arise in practical RL settings. The gap between the $\cO(\epsilon^{-4})$ guarantees in this literature and the optimal $\cO(\epsilon^{-2})$ sample complexity known for standard RL motivates the present work.

\paragraph{Reinforcement learning with general utilities.}
Our formulation is also closely related to reinforcement learning with general utilities (RLGU) \cite{NEURIPS2020_30ee748d}, where the objective is to maximize concave function $g(\lambda^\pi)$ and $\lambda^\pi \in \mathbb{R}^{|\mathcal S||\mathcal A|}$ denotes the discounted state-action occupancy measure of policy $\pi$. For finite state and action spaces, the RLGU formulation is equivalent to scalarized multi-objective RL. In particular, if we define $h_m(\lambda)=\langle r_m,\lambda\rangle$ for reward functions $r_m$ and set $g(\lambda)=f(h_1(\lambda),\dots,h_M(\lambda))$, then $h_m(\lambda^\pi)=J_m^\pi$ and the objective reduces to $f(\mathbf J^\pi)$. Conversely, by choosing $M=|\mathcal S||\mathcal A|$ and defining reward functions that correspond to indicator rewards for individual state-action pairs, the objective $f(\mathbf J^\pi)$ can be written equivalently as $g(\lambda^\pi)$. RLGU has been studied extensively in prior works \cite{pmlr-v97-hazan19a,NEURIPS2021_d7e4cdde,NEURIPS2020_30ee748d,zhang2021on,10.5555/3535850.3535906}, most of which focus on the tabular setting. More recent works extend these results to parametrized policies \cite{barakat2025on,NEURIPS2024_010c855d,JMLR:v24:22-1514} under structural assumptions such as low-rank MDPs, where the occupancy measure admits a low-dimensional representation. However, even in these regimes the best known global sample complexity for policy-gradient methods remains $\cO(\epsilon^{-4})$ \cite{barakat2025on,NEURIPS2024_010c855d}.

\section{Problem Setting}

We consider an infinite-horizon discounted Markov Decision Process (MDP)
$\mathcal{M} = (\mathcal{S}, \mathcal{A}, P, \{r_m\}_{m=1}^M, \gamma, \rho)$,
where $\mathcal{S}$ and $\mathcal{A}$ denote the state and action spaces,
$P:\mathcal{S}\times\mathcal{A}\to\Delta_{\mathcal S}$ is the transition kernel,
$r_m:\mathcal{S}\times\mathcal{A}\to[0,1]$ are the reward functions associated
with objectives $m\in[M]$, $\gamma\in(0,1)$ is the discount factor, and
$\rho\in\Delta_{\mathcal S}$ is the initial state distribution.

A stationary stochastic policy $\pi:\mathcal{S}\to\Delta_{\mathcal A}$ maps
states to distributions over actions. For each objective $m\in[M]$, the
expected discounted return under policy $\pi$ is
\begin{align}
J_m^\pi
=
\mathbb{E}_{\pi,P}
\left[
\sum_{t=0}^{\infty} \gamma^t r_m(s_t,a_t)
\right].
\end{align}

Similarly, the state-value function $V_m^\pi:\mathcal S\to\mathbb R$ and
state-action value function $Q_m^\pi:\mathcal S\times\mathcal A\to\mathbb R$
are defined as
\begin{align}
V_m^\pi(s)
&=
\mathbb{E}_{\pi,P}
\left[
\sum_{t=0}^{\infty} \gamma^t r_m(s_t,a_t)
\,\bigg|\, s_0=s
\right],\\
Q_m^\pi(s,a)
&=
\mathbb{E}_{\pi,P}
\left[
\sum_{t=0}^{\infty} \gamma^t r_m(s_t,a_t)
\,\bigg|\, s_0=s, a_0=a
\right].
\end{align}

The advantage function for objective $m$ is
$A_m^\pi(s,a)=Q_m^\pi(s,a)-V_m^\pi(s)$.
Let $\mathbf J^\pi=(J_1^\pi,\dots,J_M^\pi)^\top\in\mathbb R^M$
denote the vector of cumulative returns.

We parameterize the policy as $\pi_\theta$ with
$\theta\in\Theta\subseteq\mathbb R^d$.
The agent's goal is to solve the optimization problem
\begin{align}
\label{eq:opt_prob}
\max_{\theta\in\Theta} f(\mathbf J^{\pi_\theta}),
\end{align}
where $f:\mathbb R^M\to\mathbb R$ is a concave scalarization function.

\begin{assumption}[Concavity of $f$]
\label{assump:concave}
The function $f$ is concave. That is, for any random vector $X$ in the
domain of $f$,
\[
f(\mathbb{E}[X]) \ge \mathbb{E}[f(X)].
\]
\end{assumption}

Despite the concavity of $f$, the optimization problem
\eqref{eq:opt_prob} remains non-concave in the policy parameter $\theta$,
since each component $J_m^{\pi_\theta}$ is generally non-concave in $\theta$. We next introduce notation for derivatives of $f$.
For each coordinate $m\in[M]$, let $\partial_m f$ denote the partial
derivative of $f$ with respect to its $m$-th argument.
To facilitate convergence analysis we assume a smoothness condition.

\begin{assumption}[Smoothness of $f$]
\label{assump:smooth_f}
For each $m\in[M]$, the partial derivative $\partial_m f$
is $L_f$-Lipschitz on $\left[0,\frac{1}{1-\gamma}\right]^M$, i.e.,
\[
|\partial_m f(x)-\partial_m f(y)|
\le
L_f \|x-y\|_2
\]
for all $x,y\in\left[0,\frac{1}{1-\gamma}\right]^M$.
\end{assumption}

Since the rewards lie in $[0,1]$, each return $J_m^{\pi_\theta}$ lies in
$\left[0,\tfrac{1}{1-\gamma}\right]$. Hence it suffices to assume
Lipschitz continuity of the partial derivatives of $f$ over this compact
domain. As a consequence of this assumption, there exists $C>0$ such that $|\partial_m f(y)|
\le
C$  for all $y\in\left[0,\tfrac{1}{1-\gamma}\right]^M$ and $m\in[M]$. We further note that for $\alpha > 0$, the $\alpha$-fairness objective in \eqref{eq:alph-fairness} is not only concave but also satisfies first- and second-order smoothness on any domain such that $J_m^\pi \geq \delta > 0$ for all $m$. This positivity condition is required in any case to ensure that the objective is well-defined.

{\bf Policy Gradient and Natural Policy Gradient: } In this work we adopt a policy gradient approach to solve
\eqref{eq:opt_prob}. At iteration $k$, the parameter $\theta_k$
is updated along the gradient
$\nabla_\theta f(\mathbf J^{\pi_{\theta_k}})$.
Using the chain rule,
\begin{align}
\nabla_\theta f(\mathbf J^{\pi_\theta})
=
\sum_{m=1}^M
\partial_m f(\mathbf J^{\pi_\theta})
\nabla_\theta J_m^{\pi_\theta}.
\end{align}

Applying the policy gradient theorem \cite{sutton1999policy}
to each $\nabla_\theta J_m^{\pi_\theta}$ gives
\begin{align}
\nabla_\theta f(\mathbf J^{\pi_{\theta_k}})
=
\frac{1}{1-\gamma}
\mathbb{E}_{(s,a)\sim\nu_\rho^{\pi_{\theta_k}}}
\left[
\sum_{m=1}^M
\partial_m f(\mathbf J^{\pi_{\theta_k}})
A_m^{\pi_{\theta_k}}(s,a)
\nabla_\theta \log \pi_{\theta_k}(a|s)
\right],
\end{align}
where $\nu_\rho^{\pi_\theta}(s,a)
\coloneqq d_\rho^{\pi_\theta}(s)\pi_\theta(a|s)$
and the discounted state visitation distribution is
\begin{align}
d_\rho^{\pi_\theta}(s)
=
(1-\gamma)
\sum_{t=0}^{\infty}
\gamma^t
\mathbb P(s_t=s\mid s_0\sim\rho,\pi_\theta).
\end{align}

Note that computing this gradient requires evaluating
$\partial_m f(\mathbf J^{\pi_\theta})$, which depends on the unknown
return vector $\mathbf J^{\pi_\theta}$ and must therefore be estimated
from sampled trajectories.

Rather than using vanilla policy gradient, we employ the Natural Policy
Gradient (NPG) algorithm. At iteration $k$ the parameters are updated as
\begin{align}
\theta_{k+1}
=
\theta_k + \alpha \omega_k^*,
\end{align}
where $\alpha>0$ is the step size and
$\omega_k^*=F(\theta_k)^\dagger\nabla_\theta
f(\mathbf J^{\pi_{\theta_k}})$ is the NPG direction.
Here $F(\theta_k)$ is the Fisher Information Matrix
\begin{align}
F(\theta_k)
=
\mathbb E_{(s,a)\sim\nu_\rho^{\pi_{\theta_k}}}
\left[
\nabla_\theta \log \pi_{\theta_k}(a|s)
\nabla_\theta \log \pi_{\theta_k}(a|s)^\top
\right].
\end{align}

Directly computing the pseudoinverse $F(\theta_k)^\dagger$
can be expensive when the parameter dimension $d$ is large.
However, for fixed $\theta_k$, the direction $\omega_k^*$
is the solution of
\begin{align}
\omega_k^*
=
\arg\min_{\omega\in\mathbb R^d}
\mathcal L_{\nu_\rho^{\pi_{\theta_k}}}(\omega,\theta_k),
\end{align}
where
\begin{align}
\label{eq_def_L_nu}
\mathcal L_{\nu_\rho^{\pi_{\theta_k}}}(\omega,\theta_k)
=
\frac12
\mathbb E_{(s,a)\sim\nu_\rho^{\pi_{\theta_k}}}
\left[
\left(
\sum_{m=1}^M
\partial_m f(\mathbf J^{\pi_{\theta_k}})
A_m^{\pi_{\theta_k}}(s,a)
-
(1-\gamma)\omega^\top
\nabla_\theta\log\pi_{\theta_k}(a|s)
\right)^2
\right].
\end{align}

The function $\mathcal L_{\nu_\rho^{\pi_{\theta_k}}}$ is quadratic in
$\omega$ with gradient
$F(\theta_k)\omega-\nabla_\theta f(\mathbf J^{\pi_{\theta_k}})$.
Since both the Fisher matrix and the exact gradient are unavailable in
closed form, we estimate them from sampled trajectories and apply
stochastic gradient descent on $\mathcal L_{\nu_\rho^{\pi_{\theta_k}}}$
to obtain an approximate NPG direction $\omega_k$.
\section{Algorithm}

We develop policy-gradient algorithms for solving the optimization problem
$\max_{\theta \in \Theta} f(\mathbf J^{\pi_\theta})$. Our approach is based on Natural Policy Gradient (NPG) updates combined with estimators of the policy gradient. The key challenge lies in estimating the scalarization gradient $\partial_m f(\mathbf J^{\pi_\theta})$, since the return vector $\mathbf J^{\pi_\theta}$ must itself be estimated from sampled trajectories.

We consider two algorithmic variants depending on the regularity of the scalarization function $f$.  
When only Lipschitz continuity of $\partial_m f$ is assumed, the plug-in gradient estimator exhibits a persistent bias that prevents optimal sample complexity. To address this issue, we introduce a multi-level Monte Carlo (MLMC) estimator that controls this bias efficiently.   When $f$ satisfies an additional second-order smoothness condition, the leading-order bias cancels automatically, allowing the use of a simple empirical estimator. In this regime, vanilla Natural Policy Gradient achieves the optimal sample complexity without requiring MLMC. We first describe the general gradient estimator framework and then present the two estimators used in this work.

\subsection{Estimating \texorpdfstring{$\nabla_\theta f(\mathbf{J}^{\pi_{\theta_k}})$}{the Policy Gradient}}

Using the policy gradient theorem, the gradient of the scalarized objective can be written as
\begin{align}
\nabla_\theta f(\mathbf{J}^{\pi_{\theta_k}})
=
\mathbb{E}_{\tau \sim \pi_{\theta_k}}
\left[
\sum_{t=0}^{\infty}
\nabla_\theta \log \pi_{\theta_k}(a_t | s_t)
\left(
\sum_{m=1}^{M} \pd[m] f(\mathbf{J}^{\pi_{\theta_k}})
\sum_{h=t}^{\infty} \gamma^h r_m(s_h,a_h)
\right)
\right].
\end{align}

In practice, trajectories cannot be sampled for infinite horizons. We therefore truncate the horizon at $H$ and define the finite-horizon return vector
\begin{align}
\mathbf J_H^{\pi_\theta}
=
\left(
\mathbb{E}\!\left[\sum_{t=0}^{H-1}\gamma^t r_1(s_t,a_t)\right],
\dots,
\mathbb{E}\!\left[\sum_{t=0}^{H-1}\gamma^t r_M(s_t,a_t)\right]
\right)^\top .
\end{align}

Based on this truncation, a REINFORCE-style estimator for the gradient is given by
\begin{align}
g(\tau_i^H,\widehat{\mathbf J}_H^{\pi_\theta}|\theta)
:=
\sum_{t=0}^{H-1}
\nabla_\theta \log \pi_\theta(a_t^i|s_t^i)
\left(
\sum_{m=1}^{M}
\pd[m] f(\widehat{\mathbf J}_H^{\pi_\theta})
\sum_{h=t}^{H-1}
\gamma^h r_m(s_h^i,a_h^i)
\right),
\end{align}
where $\tau_i^H=(s_0^i,a_0^i,\dots,s_{H-1}^i,a_{H-1}^i)$ denotes a sampled trajectory and $\widehat{\mathbf J}_H^{\pi_\theta}$ is an estimator of the truncated return vector.

The choice of estimator for $\widehat{\mathbf J}_H^{\pi_\theta}$ determines the bias properties of the resulting gradient estimate. We next describe two estimators used in this work.

\subsection{Empirical Return Estimator (Vanilla NPG)}

A natural estimator for $\mathbf J_H^{\pi_\theta}$ is the empirical average of discounted returns computed from a batch of trajectories. Let $\{\tau_j^H\}_{j=1}^{B}$ denote $B$ trajectories of length $H$. The empirical return estimator is defined coordinate-wise as
\begin{align}
\label{eq:batch_J_def}
\left[\widehat{\mathbf J}_{H,B}^{\pi_\theta}\right]_m
=
\frac{1}{B}
\sum_{j=1}^{B}
\sum_{t=0}^{H-1}
\gamma^t r_m(s_t^j,a_t^j),
\quad m\in[M].
\end{align}

This estimator leads to the algorithm described in Algorithm~\ref{alg:NPG}. As we show in the analysis, when the scalarization function satisfies the second-order smoothness condition introduced in Assumption~\ref{assump:second_smooth_f}, the resulting gradient bias decays at rate $O(B^{-1})$. This rate is sufficient for vanilla Natural Policy Gradient to achieve the optimal sample complexity.

\begin{algorithm}[tbp]
\caption{Natural Policy Gradient for Joint Multi-Objective Optimization}
\label{alg:NPG}
\begin{algorithmic}[1]
\State \textbf{Initialize:} Policy parameter $\theta_0$, iteration lengths $K, N$, trajectory length $H$, batch sizes $B_1,B_2$, and step-sizes $\alpha, \beta$.
\For{$k=0,\ldots,K-1$}
\State Sample $B_1$ trajectories of length $H$ $\{\tau_i^H\}_{i=1}^{B_1} \sim \mathbb{P}_{\theta_k}^H$
\State Compute $\widehat{\mathbf J}_{H,B_1}^{\pi_{\theta_k}}$ using \eqref{eq:batch_J_def}
\State Initialize $\omega_0^k$
\For{$n=0,\ldots,N-1$}
\State Sample $B_2$ trajectories $\{\tau_i^H\}_{i=1}^{B_2}$
\State $\hat g_n^k \gets \frac{1}{B_2}\sum_{i=1}^{B_2} g(\tau_i^H,\widehat{\mathbf J}_{H,B_1}^{\pi_{\theta_k}}|\theta_k)$
\State $\widehat F_n^k \gets \frac{1}{B_2}\sum_{i=1}^{B_2}\sum_{t=0}^{H-1}\gamma^t\psi_t^{(i)}(\psi_t^{(i)})^\top$
\State $\omega_{n+1}^k \gets \omega_n^k-\beta(\widehat F_n^k\omega_n^k-\hat g_n^k)$
\EndFor
\State $\omega_k \gets \omega_N^k$
\State $\theta_{k+1}\gets\theta_k+\alpha\omega_k$
\EndFor
\end{algorithmic}
\end{algorithm}

\subsection{MLMC Return Estimator}

When only Lipschitz continuity of $\partial_m f$ is assumed, the empirical estimator leads to a gradient bias of order $O(B^{-1/2})$. Achieving sufficiently small bias would therefore require large batch sizes at every iteration, leading to suboptimal sample complexity. To overcome this limitation, we introduce a truncated multi-level Monte Carlo (MLMC) estimator that simulates large-batch gradients while requiring only logarithmic expected sampling cost.

Draw a random level $Q\sim\mathrm{Geom}(1/2)$ with $\Pr(Q=q)=2^{-q}$ for $q=1,2,\dots$. Let $\bar Q=\mathbf 1_{\{2^Q\le B_{\max}\}}Q$ denote the truncated level and sample $2^{\bar Q}$ trajectories. For each coordinate $m\in[M]$ we define
\begin{align}
\label{eq:pdv_f_MLMC}
\pd[m]f(\widehat{\mathbf J}_{H,\mathrm{MLMC}}^{\pi_\theta})
=
\pd[m]f(\widehat{\mathbf J}_{H,1}^{\pi_\theta})
+
\mathbf 1_{\{2^Q\le B_{\max}\}}2^Q
\left(
\pd[m]f(\widehat{\mathbf J}_{H,2^Q}^{\pi_\theta})
-
\pd[m]f(\widehat{\mathbf J}_{H,2^{Q-1}}^{\pi_\theta})
\right).
\end{align}

Here $\widehat{\mathbf J}_{H,2^{Q-1}}^{\pi_\theta}$ is computed using the first $2^{Q-1}$ trajectories used to construct $\widehat{\mathbf J}_{H,2^Q}^{\pi_\theta}$, ensuring nested coupling. The resulting MLMC-NPG algorithm is shown in Algorithm~\ref{alg:MLMC-NPG}.

\begin{algorithm}[tbp]
\caption{MLMC Natural Policy Gradient}
\label{alg:MLMC-NPG}
\begin{algorithmic}[1]
\State \textbf{Initialize:} Policy parameter $\theta_0$, iteration lengths $K, N$, trajectory length $H$, batch threshold $B_{\max}$, batch size $B$, and step-sizes $\alpha, \beta$.
\For{$k=0,\ldots,K-1$}
\State Draw $Q_k\sim\mathrm{Geom}(1/2)$
\State $\bar Q_k\gets Q_k\mathbf 1_{\{2^{Q_k}\le B_{\max}\}}$
\State Sample $B_k=2^{\bar Q_k}$ trajectories of length $H$ $\{\tau^H_i\}_{i=1}^{B_k} \sim \mathbb{P}_{\theta_k}^H$
\State Compute $\pd[m]f(\widehat{\mathbf J}^{\pi_{\theta_k}}_{H,\mathrm{MLMC}})$ using \eqref{eq:pdv_f_MLMC}
\State Initialize $\omega_0^k$
\For{$n=0,\ldots,N-1$}
\State Sample $B$ trajectories
\State $\hat g_n^k \gets \frac{1}{B}\sum_{i=1}^{B} g(\tau_i^H,\widehat{\mathbf J}_{H,\mathrm{MLMC}}^{\pi_{\theta_k}}|\theta_k)$
\State $\widehat F_n^k \gets \frac{1}{B}\sum_{i=1}^{B}\sum_{t=0}^{H-1}\gamma^t\psi_t^{(i)}(\psi_t^{(i)})^\top$
\State $\omega_{n+1}^k \gets \omega_n^k-\beta(\widehat F_n^k\omega_n^k-\hat g_n^k)$
\EndFor
\State $\omega_k\gets\omega_N^k$
\State $\theta_{k+1}\gets\theta_k+\alpha\omega_k$
\EndFor
\end{algorithmic}
\end{algorithm}

\subsection{Natural Policy Gradient Estimation and Policy Update}

Both algorithms compute the natural policy gradient direction by approximately solving the quadratic objective introduced in \eqref{eq_def_L_nu}. The Fisher information matrix can be estimated from a trajectory of length $H$ as
\begin{align}
\widehat F(\theta_k)
=
\sum_{t=0}^{H-1}
\gamma^t
\nabla_\theta\log\pi_{\theta_k}(a_t|s_t)
\nabla_\theta\log\pi_{\theta_k}(a_t|s_t)^\top .
\end{align}

Given a gradient estimate $\widehat g_k\approx\nabla_\theta f(\mathbf J^{\pi_{\theta_k}})$, we perform stochastic gradient descent on the quadratic objective:
\begin{align}
\omega_{n+1}^k
=
\omega_n^k
-
\beta\left(\widehat F(\theta_k)\omega_n^k-\widehat g_k\right).
\end{align}

After $N$ iterations we set $\omega_k=\omega_N^k$ and update the policy parameters according to
\begin{align}
\theta_{k+1}
=
\theta_k+\alpha\,\omega_k.
\end{align}

Algorithm~\ref{alg:NPG} and Algorithm~\ref{alg:MLMC-NPG} correspond to the vanilla and MLMC gradient estimators, respectively.
\section{Main Results}

We now present the convergence guarantees for the proposed algorithms. We begin by stating the assumptions required for the analysis.

\subsection{Assumptions}

\begin{assumption}
\label{assump:score}
The score function, $\nabla_\theta \log \pi_\theta(a|s)$, is $G_1$-Lipschitz and $G_2$-smooth. Specifically, for all $\theta,\theta_1,\theta_2\in\Theta$ and all $(s,a)\in\mathcal S\times\mathcal A$,
\begin{align*}
\|\nabla_\theta \log \pi_\theta(a|s)\| &\le G_1,\\
\|\nabla_\theta \log \pi_{\theta_1}(a|s)-\nabla_\theta \log \pi_{\theta_2}(a|s)\|
&\le G_2\|\theta_1-\theta_2\|.
\end{align*}
\end{assumption}

\begin{assumption}
\label{assump:fnd}
For all $\theta\in\mathbb R^d$, the Fisher information matrix induced by policy $\pi_\theta$ and initial distribution $\rho$ satisfies
\begin{align*}
F_\rho(\theta)
:=
\mathbb E_{s\sim d_\rho^{\pi_\theta}}
\mathbb E_{a\sim\pi_\theta(\cdot|s)}
\left[
\nabla_\theta\log\pi_\theta(a|s)
\nabla_\theta\log\pi_\theta(a|s)^\top
\right]
\succeq \mu I
\end{align*}
for some constant $\mu>0$.
\end{assumption}

\begin{assumption}
\label{assump:trans-comp-error}
Let
\begin{align*}
\mathcal L_{d_\rho^{\pi^\ast},\pi^\ast}(\omega_\theta^\ast,\theta)
\le \varepsilon_{\mathrm{bias}},
\end{align*}
for all $\theta\in\Theta$, where $\pi^\ast$ denotes the optimal policy and $\omega^*_\theta$ denotes the NPG at policy $\pi_\theta$.
\end{assumption}

\begin{assumption}[Second-Order Smoothness]
\label{assump:second_smooth_f}
For each $m\in[M]$, the partial derivative $\partial_m f$ is locally $L_{2,f}$-smooth. That is, for all $x,y\in\left[0,\frac{1}{1-\gamma}\right]^M$,
\begin{align*}
\|\nabla \partial_m f(x)-\nabla \partial_m f(y)\|
\le L_{2,f}\|x-y\|.
\end{align*}
\end{assumption}

\paragraph{Discussion of assumptions.}
Assumptions~\ref{assump:score}--\ref{assump:trans-comp-error} are extremely common in the analysis of policy-gradient methods \citep{liu2020improved,fatkhullin2023stochastic,mondal2024improved,Masiha_SCRN_KL,suttle2023beyond,ganesh2024order,ganesh2024variance}. Assumption~\ref{assump:score} ensures that the score function is bounded and Lipschitz, which is commonly used to control gradient estimation errors. Assumption~\ref{assump:fnd} requires the Fisher information matrix to be uniformly non-degenerate, a standard condition for establishing global convergence guarantees of  policy gradient methods. These conditions have  been verified for several commonly used policy classes, including Gaussian and Cauchy policies with clipped parameterized means \citep{liu2020improved,fatkhullin2023stochastic}.

Assumption~\ref{assump:trans-comp-error} captures the expressivity of the policy class and is widely used in analyses of parameterized policy-gradient methods \citep{agarwal2020optimality,fatkhullin2023stochastic,mondal2024improved,Masiha_SCRN_KL,suttle2023beyond,ganesh2024order,ganesh2024variance}. When the policy class is sufficiently expressive to represent any stochastic policy (e.g., softmax parameterization), we have $\varepsilon_{\mathrm{bias}}=0$ \citep{agarwal2021theory}. Otherwise $\varepsilon_{\mathrm{bias}}$ quantifies the approximation error induced by the restricted policy class.

Assumption~\ref{assump:second_smooth_f} is required only for our second result and enables a sharper control of the gradient bias when using the empirical return estimator.

\subsection{Main Theorems}

We now present the convergence guarantees for the two algorithmic variants described in Section~3.

\begin{theorem}[MLMC-NPG]
\label{thm:mlmc}
Let Assumptions~\ref{assump:concave}--\ref{assump:trans-comp-error} hold. Consider Algorithm~\ref{alg:MLMC-NPG} with
\[
\alpha=\frac{\mu \epsilon\log(1/\epsilon)}{4L_JG_1^2},\qquad
B_{\max}=1/\epsilon^2,\qquad
K=\Theta\!\left(\frac{1}{\alpha\epsilon}\right),\quad N =
\frac{4CMG_1}{\mu^2(1-\gamma)^2}\,
\log\!\left(\frac{R_0^2}{\epsilon^2}\right),\quad H = 2\frac{\log(1/\epsilon)}{\log(1/\gamma)}
\]
and $B=1$. Then
\begin{align}
f(\J^{\pi^\ast})
-
\frac{1}{K}\sum_{k=0}^{K-1}
\E\!\left[f(\J^{\pi_{\theta_k}})\right]
\le
\frac{\sqrt{\varepsilon_{\mathrm{bias}}}}{1-\gamma}
+
\widetilde{\mathcal O}(\epsilon).
\end{align}
\end{theorem}
As a consequence of the above theorem, the expected sample complexity required to achieve an $\epsilon$-optimal policy is $HK(\mathbb{E}[B_k] + NB) = \widetilde{\mathcal{O}}(\epsilon^{-2})$, where we have used the fact that $\mathbb{E}[B_k] = \mathcal{O}(\log B_{\max})$ (see Lemma~\ref{lem:mlmc}~$(i)$).
\begin{theorem}[Vanilla NPG under Second-Order Smoothness]
\label{thm:npg}
Let Assumptions~\ref{assump:concave}--\ref{assump:second_smooth_f} hold. Consider Algorithm~\ref{alg:NPG} with
\[
\alpha=\frac{\mu}{4L_JG_1^2},\qquad
B_1=B_2=\frac{1}{(1-\gamma)^2\epsilon},\qquad
K=\Theta\!\left(\frac{1}{\alpha\epsilon}\right),\quad N =
\frac{4CMG_1}{\mu^2(1-\gamma)^2}\,
\log\!\left(\frac{R_0^2}{\epsilon^2}\right),\quad H = 2\frac{\log(1/\epsilon)}{\log(1/\gamma)}.
\]
Then
\begin{align}
f(\J^{\pi^\ast})
-
\frac{1}{K}\sum_{k=0}^{K-1}
\E\!\left[f(\J^{\pi_{\theta_k}})\right]
\le
\frac{\sqrt{\varepsilon_{\mathrm{bias}}}}{1-\gamma}
+
{\mathcal O}(\epsilon),
\end{align}
and the resulting sample complexity is $HK(B_1+NB_2)=\widetilde{\mathcal O}(\epsilon^{-2})$.
\end{theorem}

\section{Convergence Analysis of Bias-Controlled Policy Gradient}
\subsection{Proof Overview}

Our analysis separates the optimization error of the natural policy gradient
updates from the statistical error introduced by gradient estimation.

We first derive a general convergence framework for natural policy gradient
updates that expresses the optimality gap in terms of the bias and variance
of the gradient estimator. This framework is independent of the particular
estimator used. We then analyze the bias and variance properties of the gradient estimators
considered in this work. We show that the empirical estimator suffers from a
bias of order $\mathcal{O}(B^{-1/2})$ when only Lipschitz continuity of
$\partial_m f$ is assumed, which leads to suboptimal sample complexity.

We show that this barrier can be overcome either by using a multi-level
Monte Carlo (MLMC) estimator or by exploiting second-order smoothness of the
scalarization function. Combining these bounds with the convergence framework
yields the sample complexity guarantees stated in
Theorems~\ref{thm:mlmc} and~\ref{thm:npg}.

\subsection{Convergence Framework for Natural Policy Gradient}

We begin by establishing a general framework for analyzing the convergence of
natural policy gradient updates. This result decomposes the optimality gap
into optimization and estimation components.

\begin{lemma}[General Framework]
\label{lem:general-framework}
Suppose the parameter update satisfies $\theta_{k+1} = \theta_k + \alpha \omega_k$.
Under Assumptions \ref{assump:score} and \ref{assump:trans-comp-error}, we have
\begin{align}
f(\J^{\pi^\ast}) - \frac{1}{K} \sum_{k=0}^{K-1} \mathbb{E}[f(\J^{\pi_{\theta_k}})]
\le& 
\frac{\sqrt{\varepsilon_{\mathrm{bias}}}}{1-\gamma}
+
\frac{G_1}{K} \sum_{k=0}^{K-1} \|\mathbb{E}[\omega_k] - \omega_k^\ast\|
+
\frac{G_2\alpha}{2K} \sum_{k=0}^{K-1} \mathbb{E}[\|\omega_k\|^2]
\nonumber \\
&+
\frac{1}{\alpha K}
\mathbb{E}_{s \sim d_\rho^{\pi^\ast}}
\left[
\mathrm{KL}(\pi^\ast(\cdot | s) \| \pi_{\theta_0}(\cdot | s))
\right],
\end{align}
where $\omega_k^\ast := \omega_{\theta_k}^\ast$.
\end{lemma}

We remark that a related result appears in \cite{10.1613/jair.1.13981} (see Lemma 9 therein). However, their bound is stated in terms of $\mathbb{E}\|\omega_k-\omega_k^\ast\|$, whereas ours is expressed in terms of $\|\mathbb{E}[\omega_k]-\omega_k^\ast\|$. This distinction is important for obtaining a sharp analysis, since the former quantity can be substantially looser than the latter.

To apply Lemma~\ref{lem:general-framework}, we control the second moment of
the update direction.

\begin{align}
\mathbb{E}[\|\omega_k\|^2]
\le
2\mathbb{E}[\|\omega_k^\ast\|^2]
+
2\mathbb{E}[\|\omega_k - \omega_k^\ast\|^2]
\le
2\mu \mathbb{E}[\|\nabla_\theta f(\J^{\pi_{\theta_k}})\|^2]
+
2\mathbb{E}[\|\omega_k - \omega_k^\ast\|^2].
\end{align}

Next, we bound the average squared gradient norm. Here, we use that $f(\J^{\pi_{\theta}})$ is $L_J$-smooth in $\theta$ (see Lemma \ref{lem:J_prop}). 

\begin{lemma}[Stationary Convergence]\label{lem:stationary} Suppose the parameter update satisfies $\theta_{k+1} = \theta_k + \alpha \omega_k$.
Under Assumptions \ref{assump:score}-\ref{assump:trans-comp-error}, we have
\begin{align}
\left(\frac{\alpha \mu}{2} - \alpha^2 L_J G^2 \right)
\frac{1}{K}
\sum_{k=0}^{K-1}
\E\|\nabla_\theta f(\J^{\pi_{\theta_k}})\|^2
\leq
& \,\,\frac{\mathbb{E}[f(\J^{\pi_{\theta_K}}) - f(\J^{\pi_{\theta_0}})]}{K}+
\frac{\alpha \mu}{2K}
\sum_{k=0}^{K-1}
\|\mathbb{E}[\omega_k] - \omega_k^\ast\|^2
\nonumber \\
&+
\frac{L_J \alpha^2}{K}
\sum_{k=0}^{K-1}
\mathbb{E}[\|\omega_k - \omega_k^\ast\|^2],
\end{align}
where \begin{align}
L_J = \frac{MCG_2}{(1-\gamma)^2}.
\end{align}
\end{lemma}

Finally, we bound the bias and variance terms, $\|\mathbb{E}[\omega_k] - \omega_k^\ast\|$ and $\mathbb{E}[\|\omega_k - \omega_k^\ast\|^2]$. Note that $\omega_k$ is obtained via $N$ iterations of stochastic gradient descent on a quadratic objective. By analyzing the convergence of this inner optimization loop, we obtain the following result:
\begin{lemma}[NPG Estimation Errors]
\label{lem:npg-error-bounds}
Let Assumptions \ref{assump:score} and \ref{assump:fnd} hold. Then, the estimated update direction $\omega_k$ satisfies the following error bounds:
\begin{align*}
\mathbb{E}\!\left[\|\omega_k-\omega^\ast_k\|^2\right]
\le \mathcal{O}\Bigg(
&\exp\!\left(-\frac{N\mu^2(1-\gamma)^2}{4CMG_1}\right) R_0^2
+ \frac{C^2M^2G_1^4}{\mu^2(1-\gamma)^4}\left(\frac{1}{B}+\gamma^{2H}\right)
+ \frac{C^2M^2G_1^4}{\mu^4(1-\gamma)^4}\gamma^H \\
&\qquad\qquad
+ \left(\sigma_g^2\!\left(\frac{1}{G_1^2}+\frac{1}{\mu^2}\right)\right)
\Bigg).
\end{align*}
and
\begin{align*}
\|\mathbb{E}[\omega_k]-\omega^\ast_k\|^2
\le \mathcal{O}\Bigg(
&\exp\!\left(-\frac{N\mu^2(1-\gamma)^2}{4CMG_1}\right) R_0^2
+ \frac{G_1^2\gamma^H}{\mu^2}\,R_0^2
+ \frac{G_1^6}{\mu^4}\,\gamma^H +\frac{\bar{\delta}_g^2}{\mu^2}
\Bigg).
\end{align*}
    where $\bar{\delta}_g=\left\| \nabla_\theta f(\mathbf{J}^{\pi_{\theta}}) - \mathbb{E}\left[g(\tau_i^H, \widehat{\mathbf{J}}^{\pi_\theta}_{H} | \theta)\right] \right\|$, and $\sigma_g^2=\E\left\| \nabla_\theta f(\mathbf{J}^{\pi_{\theta}}) - g(\tau_i^H, \widehat{\mathbf{J}}^{\pi_\theta}_{H} | \theta) \right\|^2$ are the bias and variance of the resulting (concave) policy gradient estimators, respectively.
\end{lemma}

By substituting Lemmas~\ref{lem:npg-error-bounds} and \ref{lem:stationary} into Lemma~\ref{lem:general-framework}, we obtain a bound characterized by the estimator's mean squared error, $\mathbb{E} \| g(\tau_i^H, \widehat{\mathbf{J}}^{\pi_\theta}_{H} | \theta) - \nabla_\theta f(\mathbf{J}^{\pi_{\theta}}) \|^2$, and the bias, $\| \nabla_\theta f(\mathbf{J}^{\pi_{\theta}}) - \mathbb{E}[g(\tau_i^H, \widehat{\mathbf{J}}^{\pi_\theta}_{H} | \theta)] \|$. The following section is dedicated to establishing formal bounds for these individual components.
\subsection{Bias and Variance of Gradient Estimators}

We now analyze the bias and variance of the gradient estimators used in our
algorithms. We distinguish between two sources of error. From Lemma \ref{lem:J_prop}, the truncation error caused by
finite trajectory length $H$ satisfies
\begin{equation}
\left\|
\mathbb{E}[g(\tau_i^H, \mathbf{J}^{\pi_\theta}_H | \theta)]
-
\nabla_\theta f(\mathbf{J}^{\pi_\theta})
\right\|
\le
\mathcal{O}(\gamma^H).
\end{equation}

A more challenging source of error arises from the empirical estimate $\widehat{\mathbf{J}}^{\pi_\theta}_H$. Specifically, the error 
\begin{equation}
    \left\| \mathbb{E}[g(\tau_i^H, \widehat{\mathbf{J}}^{\pi_\theta}_H | \theta)] - \nabla_\theta f(\mathbf{J}^{\pi_{\theta}}) \right\|
\end{equation}
decays significantly slower than the exponential rate $\gamma^H$. Unlike standard reinforcement learning where the objective is linear in the returns, the non-linearity of the map $f$ prevents the expectation operator from commuting with the partial derivatives: $\mathbb{E}[\partial_m f(\widehat{\mathbf{J}}^{\pi_\theta}_H)] \neq \partial_m f(\mathbb{E}[\widehat{\mathbf{J}}^{\pi_\theta}_H])$. This discrepancy leads to a persistent bias that depends on the sample size rather than just the horizon. Formally, we characterize this error through the following lemma:

\begin{lemma}\label{lem:grad-bias-variance}
Let Assumption \ref{assump:score} hold. Then, the bias and the second moment of the gradient estimate error are bounded as follows:
\begin{align}
    &(i)\quad \left\| \nabla_\theta f(\mathbf{J}^{\pi_{\theta}}) - \mathbb{E}[g(\tau_i^H, \widehat{\mathbf{J}}^{\pi_\theta}_H | \theta)] \right\| 
    \leq \mathcal{O} \left( \frac{G_1}{(1-\gamma)^2} \sum_{m=1}^{M} \left| \partial_m f(\mathbf{J}_H^{\pi_\theta}) - \mathbb{E} \left[ \partial_m f(\widehat{\mathbf{J}}_H^{\pi_\theta}) \right] \right| \right), \nonumber\\
    &(ii) \quad \mathbb{E} \left\| \nabla_\theta f(\mathbf{J}^{\pi_{\theta}}) - g(\tau_i^H, \widehat{\mathbf{J}}^{\pi_\theta}_H | \theta) \right\|^2 
    \leq \mathcal{O} \left( \frac{G_1^2 M}{(1-\gamma)^4} \sum_{m=1}^{M} \mathbb{E} \left| \partial_m f(\mathbf{J}_H^{\pi_\theta}) - \partial_m f(\widehat{\mathbf{J}}_H^{\pi_\theta}) \right|^2 \right).\nonumber
\end{align}
\end{lemma}

If $\partial_m f$ is non-linear but $L_f$-Lipschitz, we have:
\begin{equation}\label{eq:mid-lipschitz}
    \left| \partial_m f(\mathbf{J}_H^{\pi_\theta}) - \mathbb{E} \left[ \partial_m f(\widehat{\mathbf{J}}_H^{\pi_\theta}) \right] \right| \leq \mathbb{E} \left| \partial_m f(\mathbf{J}_H^{\pi_\theta}) - \partial_m f(\widehat{\mathbf{J}}_H^{\pi_\theta}) \right| \leq L_f \mathbb{E} \left\| \mathbf{J}_H^{\pi_\theta} - \widehat{\mathbf{J}}_H^{\pi_\theta} \right\|.
\end{equation}

\subsection{Challenge with Empirical Return Estimate}
When $\widehat{\mathbf{J}}_{H,B}^{\pi_\theta}$ is the empirical estimate formed from a batch of size $B$, it follows by definition that $\mathbb{E}[\widehat{\mathbf{J}}_{H,B}^{\pi_\theta}] = \mathbf{J}_{H}^{\pi_\theta}$. Consequently, the estimator is unbiased with respect to the finite-horizon objective. The variance of $\widehat{\mathbf{J}}_{H,B}^{\pi_\theta}$ is bounded as follows:
\begin{lemma}\label{lem:hat-J-variance}
The mean squared error of the empirical return estimate satisfies:
\begin{equation}
    \mathbb{E} \left\| \widehat{\mathbf{J}}_{H,B}^{\pi_\theta} - \mathbf{J}_{H}^{\pi_\theta} \right\|^2 \leq \frac{M}{(1-\gamma)^2 B}.\nonumber
\end{equation}
\end{lemma}
\begin{proof}
Let $\mathbf{r}(s,a) := (r_1(s,a),\ldots,r_M(s,a))^\top \in \mathbb{R}^M$. Then
\begin{align*} \mathbb{E} \left\| \widehat{\mathbf{J}}_{H,B}^{\pi_\theta} - \mathbf{J}_{H}^{\pi_\theta} \right\|^2 &= \mathbb{E} \norm{ \frac{1}{B} \sum_{j=1}^{B} \sum_{t=0}^{H-1} \gamma^t \mathbf{r}(s_t^j,a_t^j) - \E\left[ \sum_{t=0}^{H-1} \gamma^t \mathbf{r}(s_t^j,a_t^j)\right] }^2 \\ &\overset{(a)}{=} \frac{1}{B^2}\sum_{j=1}^{B}\mathbb{E} \norm{ \left[\sum_{t=0}^{H-1} \gamma^t \mathbf{r}(s_t^j,a_t^j) \right]- \E\left[ \sum_{t=0}^{H-1} \gamma^t \mathbf{r}(s_t^j,a_t^j)\right] }^2\\ &= \frac{1}{B^2}\sum_{j=1}^{B}\mathbb{E} \norm{ \sum_{t=0}^{H-1} \gamma^t \left(\mathbf{r}(s_t^j,a_t^j)-\E[\mathbf{r}(s_t^j,a_t^j)]\right) }^2\\ &\overset{(b)}{\leq} \frac{M}{(1-\gamma)^2 B}, \end{align*}
Here, $(a)$ uses the independence of the $B$ trajectories, and $(b)$ uses $r_m(s,a)\in[0,1]$ for all $m\in[M]$ and $(s,a)\in\mathcal{S}\times\mathcal{A}$ (hence $\|\sum_{t=0}^{H-1}\gamma^t(\mathbf{r}-\mathbb{E}\mathbf{r})\|^2 \le M(\sum_{t=0}^{H-1}\gamma^t)^2 \le M/(1-\gamma)^2$).

\end{proof}
Combining Lemmas \ref{lem:grad-bias-variance} and \ref{lem:hat-J-variance} with Equation \eqref{eq:mid-lipschitz}, the bias of the gradient estimate can be bounded as follows:
\begin{align}\label{eq:B-bias}
    \left\| \nabla_\theta f(\mathbf{J}^{\pi_{\theta}}) - \mathbb{E}\left[g(\tau_i^H, \widehat{\mathbf{J}}^{\pi_\theta}_{H,B} | \theta)\right] \right\| &\leq \frac{G_1}{(1-\gamma)^2} \sum_{m=1}^{M} \E \left| \partial_m f(\mathbf{J}_H^{\pi_\theta}) -  \partial_m f(\widehat{\mathbf{J}}_{H,B}^{\pi_\theta}) \right|  \nonumber\\
&\leq \frac{G_1ML_f}{(1-\gamma)^2}  \E \norm{\mathbf{J}_{H}^{\pi_\theta} -  \widehat{\mathbf{J}}_{H,B}^{\pi_\theta} } \nonumber \\
    &\leq \mathcal{O} \left( \frac{G_1 M^{3/2}L_f}{(1-\gamma)^3 \sqrt{B}} \right).
\end{align}

Consequently, to ensure this bias is suppressed to $\epsilon$, the batch size $B$ must be at least $\mathcal{O}(\epsilon^{-2})$. Since this batch size is required within each policy iteration, achieving an overall sample complexity of $\mathcal{O}(\epsilon^{-2})$ is not directly possible with this vanilla approach. This is the key that leads to the result of $\mathcal{O}(\epsilon^{-4})$ in the previous works \cite{10.1613/jair.1.13981}. 

In standard RL, this issue does not arise. Indeed, standard RL is recovered as a special case of our formulation with a single objective $J^\pi$ and scalarization $f(x)=x$. Since $f$ is linear, by Lemma~\ref{lem:grad-bias-variance},
\begin{align}
    \left\| \nabla_\theta f(J^{\pi_{\theta}}) - \mathbb{E}\!\left[g(\tau_i^H, \widehat{J}_H^{\pi_\theta}\mid \theta)\right] \right\|
    \leq
    \mathcal{O}\!\left(
    \frac{G_1}{(1-\gamma)^2}
    \left|
    f'(J_H^{\pi_\theta})
    -
    \mathbb{E}\!\left[f'(\widehat{J}_H^{\pi_\theta})\right]
    \right|
    \right).
\end{align}
Since $f'$ is constant for linear $f$, the right-hand side is zero. Therefore, the estimator is unbiased in the standard RL setting.

\subsection{MLMC Gradient Estimator}

To improve upon the prohibitive sample complexity of the vanilla estimator, we employ a Multi-level Monte Carlo (MLMC) approach. This technique allows us to effectively reduce the bias of the non-linear gradient estimate by constructing a telescoping sum of estimators with increasing batch sizes.

\begin{lemma}
\label{lem:mlmc}
Consider the MLMC estimator $g(\tau_i^H, \widehat{\mathbf{J}}^{\pi_\theta}_{H, \mathrm{MLMC}} | \theta)$, truncated at $B_{\max}$. Under Assumptions \ref{assump:concave}-\ref{assump:score}, the following bounds on the expected number of trajectories, bias and variance hold:
\begin{align*}
&(i)\quad\E[B_k]=\mathcal{O}(\log B_{\max})\\
&(ii) \quad   \left\| \nabla_\theta f(\mathbf{J}^{\pi_{\theta}}) - \mathbb{E}\left[g(\tau_i^H, \widehat{\mathbf{J}}^{\pi_\theta}_{H, \mathrm{MLMC}} | \theta)\right] \right\| 
    \leq \mathcal{O} \left( \frac{G_1 M^{3/2}}{(1-\gamma)^3 \sqrt{B_{\max}}} \right), \\
   &(iii)\quad \mathbb{E} \left\| \nabla_\theta f(\mathbf{J}^{\pi_{\theta}}) - g(\tau_i^H, \widehat{\mathbf{J}}^{\pi_\theta}_{H, \mathrm{MLMC}} | \theta) \right\|^2 
    \leq \mathcal{O} \left( \frac{L_f^2 G_1^2 M^3 \log B_{\max}}{(1-\gamma)^6} \right).
\end{align*}
\end{lemma}
\begin{proof}
    Let $Q$ be a geometric random variable with $\Pr(Q=q)=2^{-q}$ for $q \in \{1, 2, \dots, J\}$, where $J = \lfloor \log_2 B_{\max} \rfloor$. We first prove part $(i)$. We compute the expected number of trajectories sampled. The number of sampled trajectories equals $2^Q$ on the event $\{2^Q\le B_{\max}\}$ and
equals $1$ otherwise. Hence,
\begin{align}
\E\big[B_k\big]
&= \E\left[2^Q\,\mathbf{1}_{\{2^Q\le B_{\max}\}}\right] + \E\left[\mathbf{1}_{\{2^Q> B_{\max}\}}\right] \notag\\
&= \sum_{q=1}^{\lfloor\log_2 B_{\max}\rfloor} \Pr(Q=q)\,2^q + \Pr(Q>\lfloor\log_2 B_{\max}\rfloor) \notag\\
&= \sum_{q=1}^{\lfloor\log_2 B_{\max}\rfloor} 2^{-q}\,2^q + 2^{-\lfloor\log_2 B_{\max}\rfloor} \notag\\
&= \cO(\lfloor\log_2 B_{\max}\rfloor).
\end{align}
Next for part $(ii)$, taking the expectation of the MLMC estimator for the partial derivative $\partial_m f$ yields:
\begin{align}
\mathbb{E}\left[\partial_m f\left(\widehat{\mathbf{J}}_{H,\mathrm{MLMC}}^{\pi_\theta}\right)\right]
&= \mathbb{E}\left[\partial_m f\left(\widehat{\mathbf{J}}_{H,1}^{\pi_\theta}\right)\right] + \sum_{q=1}^{J} \Pr(Q=q) \cdot 2^{q} \mathbb{E}\left[ \partial_m f\left(\widehat{\mathbf{J}}_{H,2^{q}}^{\pi_\theta}\right) - \partial_m f\left(\widehat{\mathbf{J}}_{H,2^{q-1}}^{\pi_\theta}\right) \right] \notag \\
&= \mathbb{E}\left[\partial_m f\left(\widehat{\mathbf{J}}_{H,1}^{\pi_\theta}\right)\right] + \sum_{q=1}^{J} \left( \mathbb{E}\left[\partial_m f\left(\widehat{\mathbf{J}}_{H,2^{q}}^{\pi_\theta}\right)\right] - \mathbb{E}\left[\partial_m f\left(\widehat{\mathbf{J}}_{H,2^{q-1}}^{\pi_\theta}\right)\right] \right) \notag \\
&= \mathbb{E}\left[\partial_m f\left(\widehat{\mathbf{J}}_{H,2^{J}}^{\pi_\theta}\right)\right].
\end{align}

The telescoping property ensures that the expectation of the MLMC estimator is identical to the expectation of the vanilla estimator with a batch size of $2^J$. Specifically, when $B_{\max}$ is a power of two, the right-hand side reduces to $\mathbb{E}\left[\partial_m f\left(\widehat{\mathbf{J}}_{H,B_{\max}}^{\pi_\theta}\right)\right]$.

By applying the bounds derived in Lemmas \ref{lem:vanilla} $(i)$, we obtain the following bound for the gradient bias:
\begin{equation}
    \left\| \nabla_\theta f(\mathbf{J}^{\pi_{\theta}}) - \mathbb{E}\left[g(\tau_i^H, \widehat{\mathbf{J}}^{\pi_\theta}_{H, \mathrm{MLMC}} | \theta)\right] \right\| \leq \mathcal{O} \left( \frac{G_1 M^{3/2}}{(1-\gamma)^3 \sqrt{2^J}} \right).
\end{equation}
 We reserve the proof of part $(iii)$ to Appendix \ref{app:mlmc}.
\end{proof}

\subsection{Improved Bias under Second-Order Smoothness}

For non-linear objective functions, the plug-in estimator $\partial_m f(\widehat{\mathbf{J}}_{H,B}^{\pi_\theta})$ is generally biased. While Lipschitz continuity only guarantees a bias reduction rate of $\mathcal{O}(B^{-1/2})$, we can achieve a more favorable rate of $\mathcal{O}(B^{-1})$ by leveraging second-order smoothness.

\begin{lemma}
\label{lem:vanilla}
Consider the empirical average estimator $g(\tau_i^H, \widehat{\mathbf{J}}^{\pi_\theta}_{H, B} | \theta)$. Under Assumptions \ref{assump:concave}-\ref{assump:score} and \ref{assump:second_smooth_f}, the following bias and variance bounds hold:
\begin{align*}
    &(i) \quad \left\| \nabla_\theta f(\mathbf{J}^{\pi_{\theta}}) - \mathbb{E}\left[g(\tau_i^H, \widehat{\mathbf{J}}^{\pi_\theta}_{H, B} | \theta)\right] \right\| 
    \leq \mathcal{O} \left( \frac{L_{2,f}G_1 M^{2}}{(1-\gamma)^4 B} \right)\\
     &(ii) \quad \mathbb{E} \left\| \nabla_\theta f(\mathbf{J}^{\pi_{\theta}}) - g(\tau_i^H, \widehat{\mathbf{J}}^{\pi_\theta}_{H, B} | \theta) \right\|^2 
    \leq \mathcal{O} \left( \frac{L_f^2 G_1^2 M^3 }{(1-\gamma)^6B} \right).
\end{align*}
\end{lemma}
The proof of the above result can be found in Appendix \ref{app:vanilla}.

\begin{table}[t]
\centering
\caption{Comparison of the bias, variance, and expected number of trajectories for the gradient estimators analyzed in this work.}
\label{tab:estimator_comparison}
\begin{tabular}{lccc}
\hline
Estimator & Expected $\#$ of trajectories & Variance & Bias \\ \hline
Empirical (Lipschitz $\partial_m f$)  
& $B$ 
& $\mathcal{O}(1/B)$ 
& $\mathcal{O}(1/\sqrt{B})$ \\

Empirical (smooth $\partial_m f$)     
& $B$ 
& $\mathcal{O}(1/B)$ 
& $\mathcal{O}(1/B)$ \\

MLMC (Lipschitz $\partial_m f$)        
& $\mathcal{O}(\log B_{\max})$ 
& $\mathcal{O}(\log B_{\max})$ 
& $\mathcal{O}(1/\sqrt{B_{\max}})$ \\

MLMC (smooth $\partial_m f$)           
& $\mathcal{O}(\log B_{\max})$ 
& $\mathcal{O}(\log B_{\max})$ 
& $\mathcal{O}(1/B_{\max})$ \\

\hline
\end{tabular}
\end{table}

Table~\ref{tab:estimator_comparison} highlights the fundamental
difficulty in concave multi-objective reinforcement learning. 
When only Lipschitz continuity of $\partial_m f$ is assumed,
the empirical estimator exhibits a $\mathcal{O}(B^{-1/2})$ bias,
which prevents policy-gradient methods from achieving the
optimal $\mathcal{O}(\epsilon^{-2})$ sample complexity, and
leads to the suboptimal $\mathcal{O}(\epsilon^{-4})$ sample complexity
observed in existing policy-gradient analyses \cite{10.1613/jair.1.13981}.

The MLMC estimator circumvents this limitation by effectively
simulating large-batch gradient estimates while requiring only
logarithmic sampling cost. Alternatively, when the scalarization
function satisfies the second-order smoothness assumption,
the leading-order bias term cancels, allowing the empirical
estimator to achieve the faster $\mathcal{O}(B^{-1})$ bias decay. These two mechanisms form the basis of the optimal sample
complexity guarantees established in Theorems~\ref{thm:mlmc}
and~\ref{thm:npg}.

\section{Conclusion}


We study \emph{concave-scalarized multi-objective reinforcement learning} in a
policy-gradient framework and establish optimal sample complexity
guarantees for this setting. A key challenge arises from the
nonlinearity of the scalarization function: estimating the return
vector from trajectories introduces bias in the policy-gradient
estimator, which fundamentally limits standard analyses. We show that this barrier can be overcome by controlling the bias of
the gradient estimator. In particular, Natural Policy Gradient (NPG)
combined with a multi-level Monte Carlo (MLMC) estimator achieves the
optimal $\cO(\epsilon^{-2})$ sample complexity without additional
assumptions. Moreover, when the scalarization is twice continuously
differentiable, we show that vanilla NPG attains the same rate through
a refined bias analysis.


\bibliography{references}
\bibliographystyle{abbrvnat}
\appendix
\newpage


\section{Proof of Lemma \ref{lem:general-framework}}
\label{app:general-framework}
We begin our analysis by restating the performance difference lemma for concave utilities, which serves as a foundational tool for our derivation.

\begin{lemma}[Performance difference lemma for concave utility]\label{lem:perf-difference}
The difference in performance between any two policies $\pi_\theta$ and $\pi_{\theta'}$ is bounded as follows:
\begin{equation}
    (1 - \gamma) \left[ f(\mathbf{J}^{\pi_\theta}) - f(\mathbf{J}^{\pi_{\theta'}}) \right] \le \sum_{m=1}^{M} \partial_m f(\mathbf{J}^{\pi_{\theta'}}) \mathbb{E}_{s \sim d_\rho^{\pi_\theta}} \mathbb{E}_{a \sim \pi_\theta(\cdot|s)} \left[ A_m^{\pi_{\theta'}}(s, a) \right].
\end{equation}
\end{lemma}

This result is adapted from \cite{10.1613/jair.1.13981} (specifically Lemma 8). 

Next, we characterize the progress of the policy update by considering the expected change in the Kullback–Leibler (KL) divergence:
\begin{align}
&\mathbb{E}_{s\sim d_\rho^{\pi^*}}\!\Big[
\KL(\pi^*(\cdot| s)\|\pi_{\theta_k}(\cdot| s))
- \KL(\pi^*(\cdot| s)\|\pi_{\theta_{k+1}}(\cdot| s))
\Big] \nonumber\\
&= \mathbb{E}_{s\sim d_\rho^{\pi^*}}  \mathbb{E}_{a\sim \pi^*(\cdot| s)}
\Bigg[ \log \frac{\pi_{\theta_{k+1}}(a| s)}{\pi_{\theta_k}(a| s)} \Bigg] \nonumber\\
&\overset{(a)}{\ge}
\mathbb{E}_{s\sim d_\rho^{\pi^*}}  \mathbb{E}_{a\sim \pi^*(\cdot| s)}
\Big[ \nabla_\theta \log \pi_{\theta_k}(a| s)\cdot(\theta_{k+1}-\theta_k) \Big]
-\frac{G_2}{2}\|\theta_{k+1}-\theta_k\|^2 \nonumber\\
&= \alpha\mathbb{E}_{s\sim d_\rho^{\pi^*}}  \mathbb{E}_{a\sim \pi^*(\cdot| s)}
\Big[ \nabla_\theta \log \pi_{\theta_k}(a| s)\cdot \omega_k \Big]
-\frac{G_2\alpha^2}{2}\|\omega_k\|^2 \nonumber\\
&= \alpha\mathbb{E}_{s\sim d_\rho^{\pi^*}}  \mathbb{E}_{a\sim \pi^*(\cdot| s)}
\Big[ \nabla_\theta \log \pi_{\theta_k}(a| s)\cdot \omega^*_k \Big]
+\alpha\mathbb{E}_{s\sim d_\rho^{\pi^*}}  \mathbb{E}_{a\sim \pi^*(\cdot| s)}
\Big[ \nabla_\theta \log \pi_{\theta_k}(a| s)\cdot(\omega_k-\omega^*_k) \Big] \nonumber\\
&\hspace{3.2cm}
-\frac{G_2\alpha^2}{2}\|\omega_k\|^2 \nonumber\\
&= \alpha\big[f(\J^*)-f(\J^{\pi_{\theta_k}})\big]
+\alpha\mathbb{E}_{s\sim d_\rho^{\pi^*}}  \mathbb{E}_{a\sim \pi^*(\cdot| s)}
\Big[ \nabla_\theta \log \pi_{\theta_k}(a| s)\cdot \omega^*_k \Big]
-\alpha\big[f(\J^*)-f(\J^{\pi_{\theta_k}})\big] \nonumber\\
&\hspace{3.2cm}
+\alpha\mathbb{E}_{s\sim d_\rho^{\pi^*}}  \mathbb{E}_{a\sim \pi^*(\cdot| s)}
\Big[ \nabla_\theta \log \pi_{\theta_k}(a| s)\cdot(\omega_k-\omega^*_k) \Big]
-\frac{G_2 \alpha^2}{2}\|\omega_k\|^2 \nonumber\\
&\overset{(b)}{=}
\alpha\big[f(\J^*)-f(\J^{\pi_{\theta_k}})\big]
+\alpha\mathbb{E}_{s\sim d_\rho^{\pi^*}}  \mathbb{E}_{a\sim \pi^*(\cdot| s)}
\Big[ \nabla_\theta \log \pi_{\theta_k}(a| s)\cdot(\omega_k-\omega^*_k) \Big]
-\frac{G_2\alpha^2}{2}\|\omega_k\|^2\nonumber\\
&\hspace{2.4cm}
+\frac{\alpha}{1-\gamma}
\mathbb{E}_{s\sim d_\rho^{\pi^*}}  \mathbb{E}_{a\sim \pi^*(\cdot| s)}
\Bigg[
\nabla_\theta \log \pi_{\theta_k}(a| s)\cdot
\Bigg( (1-\gamma)\omega^*_k
-\sum_{m=1}^{M}\pd[m]f(\J^{\pi_{\theta_k}})
A_m^{\pi_{\theta_k}}(s,a) \Bigg)
\Bigg]  \nonumber\\
&\overset{(c)}{\ge}
\alpha\big[f(\J^{\pi^*})-f(\J^{\pi_{\theta_k}})\big]
+\alpha\mathbb{E}_{s\sim d_\rho^{\pi^*}}  \mathbb{E}_{a\sim \pi^*(\cdot| s)}
\Big[ \nabla_\theta \log \pi_{\theta_k}(a| s)\cdot(\omega_k-\omega^*_k) \Big]
-\frac{G_2\alpha^2}{2}\|\omega_k\|^2  \nonumber\\
&\hspace{1.2cm}
-\frac{\alpha}{1-\gamma}
\sqrt{
\mathbb{E}_{s\sim d_\rho^{\pi^*}}  \mathbb{E}_{a\sim \pi^*(\cdot| s)}
\Bigg[
\Bigg(
\nabla_\theta \log \pi_{\theta_k}(a| s)\cdot
\Bigg( (1-\gamma)\omega^*_k
-\sum_{m=1}^{M}\pd[m]f(\J^{\pi_{\theta_k}})
A_m^{\pi_{\theta_k}}(s,a) \Bigg)
\Bigg)^2
\Bigg]
}
\end{align}

where step $(a)$ holds by Assumption \ref{assump:score} and step $(b)$ holds by Lemma \ref{lem:perf-difference}.
Step $(c)$ uses the convexity of the function $f(x)=x^2$. Rearranging items, we have
\begin{align}
f(\J^*) - \E f(\J^{\pi_{\theta_k}})
\le\;&
\frac{\sqrt{\epsilon_{\mathrm{bias}}}}{1-\gamma}
+ G_1\|\E[\omega_k]-\omega^*_k\|
+ \frac{G_2\alpha}{2}\E\|\omega_k\|^2 \nonumber\\
&\quad
+\frac{1}{\alpha}
\E \mathbb{E}_{s\sim d_\rho^{\pi^*}}\!\Big[
\KL(\pi^*(\cdot| s)\|\pi_{\theta_k}(\cdot| s))
- \KL(\pi^*(\cdot| s)\|\pi_{\theta_{k+1}}(\cdot| s))
\Big].
\end{align}

Summing from $k=0$ to $K-1$ and dividing by $K$ concludes the proof of Lemma \ref{lem:general-framework}.

\section{Proof of Lemma \ref{lem:stationary}}
\label{app:stationary}
To analyze the impact of horizon truncation, we restate the following result from \cite{10.1613/jair.1.13981} (Lemma 6 and 7 therein), which characterizes the smoothness and gradient error of the objective functions.

\begin{lemma}\label{lem:J_prop}
Suppose Assumptions \ref{assump:smooth_f} and \ref{assump:score} hold. Both the infinite-horizon objective $f(\mathbf{J}^{\pi_\theta})$ and the truncated objective $f(\mathbf{J}_H^{\pi_\theta})$ are $L_J$-smooth in $\theta$, where
\begin{align}
L_J = \frac{MCG_2}{(1-\gamma)^2}.
\end{align}
Furthermore, the gradients of the original and truncated objectives satisfy
\begin{align}
\|\nabla_\theta f(\mathbf{J}^{\pi_\theta}) - \nabla_\theta f(\mathbf{J}_H^{\pi_\theta})\| \le \frac{MG_1\gamma^H}{(1-\gamma)^2} \left[ \sqrt{M}L_f \frac{1-\gamma^H - H\gamma^H(1-\gamma)}{1-\gamma} + C(1+H(1-\gamma)) \right].
\end{align}
\end{lemma}

Using the $L_J$-smoothness of $f(\J^{\pi_\theta})$ with respect to $\theta$ (Lemma~\ref{lem:J_prop}), we obtain
\begin{align}
f(\J^{\pi_{\theta_{k+1}}})
&\ge f(\J^{\pi_{\theta_k}})
+ \left\langle \nabla_\theta f(\J^{\pi_{\theta_k}}),\, \theta_{k+1}-\theta_k \right\rangle
- \frac{L_J}{2}\|\theta_{k+1}-\theta_k\|^2 \nonumber\\
&= f(\J^{\pi_{\theta_k}})
+ \alpha \left\langle \nabla_\theta f(\J^{\pi_{\theta_k}}),\, \omega_k \right\rangle
- \frac{L_J\alpha^2}{2}\|\omega_k\|^2 \nonumber\\
&= f(\J^{\pi_{\theta_k}})
+ \alpha \left\langle \nabla_\theta f(\J^{\pi_{\theta_k}}),\, \omega_k-\omega_k^\ast+\omega_k^\ast \right\rangle
- \frac{L_J\alpha^2}{2}\left\|\omega_k-\omega_k^\ast+\omega_k^\ast\right\|^2 \nonumber\\
&\ge f(\J^{\pi_{\theta_k}})
+ \alpha \left\langle \nabla_\theta f(\J^{\pi_{\theta_k}}),\, \omega_k^\ast \right\rangle
+ \alpha \left\langle \nabla_\theta f(\J^{\pi_{\theta_k}}),\, \omega_k-\omega_k^\ast \right\rangle \nonumber\\
&\hspace{1.5cm}
- L_J\alpha^2\left(\|\omega_k-\omega_k^\ast\|^2+\|\omega_k^\ast\|^2\right) \nonumber\\
&\overset{(a)}{\ge}
f(\J^{\pi_{\theta_k}})
+ \alpha\mu \|\nabla_\theta f(\J^{\pi_{\theta_k}})\|^2
+ \alpha \left\langle \nabla_\theta f(\J^{\pi_{\theta_k}}),\, \omega_k-\omega_k^\ast \right\rangle \nonumber\\
&\hspace{1.5cm}
- L_J\alpha^2\|\omega_k-\omega_k^\ast\|^2
- L_J\alpha^2\|\omega_k^\ast\|^2 \nonumber\\
&\overset{(b)}{\ge}
f(\J^{\pi_{\theta_k}})
+ (\alpha\mu-\alpha^2L_JG_1^2)\|\nabla_\theta f(\J^{\pi_{\theta_k}})\|^2
+ \alpha \left\langle \nabla_\theta f(\J^{\pi_{\theta_k}}),\, \omega_k-\omega_k^\ast \right\rangle \nonumber\\
&\hspace{1.5cm}
- L_J\alpha^2\|\omega_k-\omega_k^\ast\|^2 .
\end{align}

Taking conditional expectation given $\theta_k$, we have
\begin{align}
\E\!\left[f(\J^{\pi_{\theta_{k+1}}}) | \theta_k\right]
&\ge
f(\J^{\pi_{\theta_k}})
+ (\alpha\mu-\alpha^2L_JG_1^2)\|\nabla_\theta f(\J^{\pi_{\theta_k}})\|^2 \nonumber\\
&\hspace{0.7cm}
+ \alpha \left\langle \nabla_\theta f(\J^{\pi_{\theta_k}}),\, \E[\omega_k|\theta_k]-\omega_k^\ast \right\rangle
- L_J\alpha^2 \E\!\left[\|\omega_k-\omega_k^\ast\|^2 | \theta_k\right] \nonumber\\
&\ge
f(\J^{\pi_{\theta_k}})
+ (\alpha\mu-\alpha^2L_JG_1^2)\|\nabla_\theta f(\J^{\pi_{\theta_k}})\|^2 \nonumber\\
&\hspace{0.7cm}
- \alpha\left(
c\|\nabla_\theta f(\J^{\pi_{\theta_k}})\|^2
+ \frac{1}{c}\|\E[\omega_k|\theta_k]-\omega_k^\ast\|^2
\right) \nonumber\\
&\hspace{0.7cm}
- L_J\alpha^2 \E\!\left[\|\omega_k-\omega_k^\ast\|^2 | \theta_k\right] \nonumber\\
&=
f(\J^{\pi_{\theta_k}})
+ (\alpha\mu-\alpha^2L_JG_1^2-c\alpha)\|\nabla_\theta f(\J^{\pi_{\theta_k}})\|^2 \nonumber\\
&\hspace{0.7cm}
- \frac{\alpha}{c}\|\E[\omega_k|\theta_k]-\omega_k^\ast\|^2
- L_J\alpha^2 \E\!\left[\|\omega_k-\omega_k^\ast\|^2 | \theta_k\right].
\end{align}

Here, step $(a)$ follows from Assumption \ref{assump:fnd}, and step $(b)$ follows from Assumption \ref{assump:score}.
The second inequality after taking conditional expectation follows from Young's inequality.
We obtain the following by choosing $c=\mu/2$ and taking expectation over $\theta_k$:
\begin{align}
\E\!\left[f(\J^{\pi_{\theta_{k+1}}}) \right]
&\ge
\E[f(\J^{\pi_{\theta_k}})]
+ \left(\frac{\alpha\mu}{2}-\alpha^2L_JG_1^2\right)\E\|\nabla_\theta f(\J^{\pi_{\theta_k}})\|^2 \nonumber\\
&\hspace{0.7cm}
- \frac{2\alpha}{\mu}\E\|\E[\omega_k|\theta_k]-\omega_k^\ast\|^2
- L_J\alpha^2 \E\!\left[\|\omega_k-\omega_k^\ast\|^2\right].
\end{align}
Finally, rearranging and taking average over $k=1,\cdots,K-1$ yields Lemma~\ref{lem:stationary}.

\section{Proof of Lemma \ref{lem:npg-error-bounds}}
\label{app:npg-error-bounds}

The Natural Policy Gradient (NPG) subroutine within our algorithm can be formally analyzed as an instance of a stochastic linear recursion. We consider a recursion of the following form:
\begin{align}\label{eq:omega_n}
\omega_{n+1} = \omega_n - \beta (\widehat{F}_n \omega_n - \widehat{g}_n),
\end{align}
where $\widehat{F}_n \in \mathbb{R}^{d \times d}$ and $\widehat{g}_n \in \mathbb{R}^d$ are noisy estimates of the matrix $F$ and the gradient vector $\nabla_\theta f(\mathbf{J}^{\pi_\theta})$, respectively, for $n \in \{0, \dots, N-1\}$. For simplicity of notation, we omit the superscript $k$ for the policy parameter $\theta_k$ within this subroutine analysis. Moreover, we will denote the Fisher information matrix $F(\theta_k)$ as simply $F$. We denote $\mathbb{E}_n[\cdot]$ as the conditional expectation given the history up to step $n$. We assume the following conditions hold for all $n$:
\newline

\noindent \textit{Conditions on Oracle Errors:}
\begin{itemize}
    \item $(A_1, A_2)$: $\mathbb{E}_n[\|\widehat{F}_n - F\|^2] \leq \sigma_F^2$ and $\|\mathbb{E}_n[\widehat{F}_n] - F\|^2 \leq \delta_F^2$.
    \item $(A_3, A_4)$: $\mathbb{E}_n[\|\widehat{g}_n - \nabla_\theta f(\mathbf{J}^{\pi_\theta})\|^2] \leq \sigma_g^2$ and $\|\mathbb{E}_n[\widehat{g}_n] - \nabla_\theta f(\mathbf{J}^{\pi_\theta})\|^2 \leq \delta_g^2$.
    \item $(A_5)$: $\|\mathbb{E}[\widehat{g}_n] - \nabla_\theta f(\mathbf{J}^{\pi_\theta})\|^2 \leq \bar{\delta}_g^2$, where $\bar{\delta}_g^2 \leq \delta_g^2$ by Jensen's inequality.
\end{itemize}

\noindent \textit{Regularity Conditions:}
\begin{itemize}
    \item $(A_6, A_7)$: $\|F\| \leq \Lambda_F$ and $\|\nabla_\theta f(\mathbf{J}^{\pi_\theta})\| \leq \Lambda_g$.
    \item $(A_8)$: $\omega^\top F \omega \geq \mu \|\omega\|^2$ for all $\omega \in \mathbb{R}^d$.
\end{itemize}

Let $\omega^\ast = F^\dagger \nabla_\theta f(\mathbf{J}^{\pi_\theta})$. We demonstrate that these conditions hold for our specific update rules. Under these requirements, we can apply Theorem 2 in \cite{ganesh2025regret} to characterize the convergence of the recursion in both mean-square and bias error, as follows:

\begin{theorem}
\label{thm_2}
Consider the recursion \eqref{eq:omega_n} and suppose $(A_1)-(A_8)$ hold. Let the step size be $\beta = \frac{\mu}{\Lambda_F}$ and assume the bias satisfies $\delta_F \leq \mu / 8$. Then, for any $N \geq 1$, the following bounds hold:
    \begin{align*}
    \mathbb{E}\left[\|\omega_N - \omega^\ast\|^2\right] \leq \mathcal{O} \Bigg( \exp\left(-\tfrac{N \mu^2}{4\Lambda_g}\right)R_0^2 + \frac{\sigma_F^2 \Lambda_g^2}{\mu^2 \Lambda_F} + \frac{\sigma_g^2}{\Lambda_F} + \frac{\delta_F^2 \Lambda_g^2}{\mu^4} + \frac{\delta_g^2}{\mu^2} \Bigg).
    \end{align*}
and
\begin{align*}
    \|\mathbb{E}[\omega_N] - \omega^\ast\|^2 \leq \mathcal{O} \Bigg( \exp\left(-\tfrac{N \mu^2}{4\Lambda_g}\right) R_0^2 + \frac{\delta_F^2 R_0^2}{\mu^2} + \frac{\delta_F^2 \Lambda_F^2}{\mu^4} + \frac{\bar{\delta}_g^2}{\mu^2} \Bigg),
    \end{align*}
where $R_0 = \|\omega_0 - \omega^\ast\|^2$.
\end{theorem}

    Recall that $F$ and $\widehat{F}_n$ are given by:
    \begin{align}
        F &= (1-\gamma)\mathbb{E}_{\tau} \left[\sum_{t=0}^\infty\gamma^t \nabla_\theta \log \pi_{\theta_k}(a_t | s_t) \nabla_\theta \log \pi_{\theta_k}(a_t | s_t)^\top \right] \\
        \widehat{F} &= \frac{(1-\gamma)}{B}\sum_{b=1}^B\left[\sum_{t=0}^{H-1}\gamma^t \nabla_\theta \log \pi_{\theta_k}(a_t^b | s_t^b) \nabla_\theta \log \pi_{\theta_k}(a_t^b | s_t^b)^\top \right]
    \end{align}

    Under Assumption \ref{assump:score}, we have $\|\nabla_\theta \log \pi_{\theta_k}(a|s)\| \leq G_1$, as a result the spectral norm of $F$ is bounded as $\|F\| \leq G_1^2=:\Lambda_F$. The truncation error leads to a bias bound of:
    \begin{equation}
        \|\mathbb{E}[\widehat{F}]-F\| = \left\| (1-\gamma) \mathbb{E} \left[\sum_{t=H}^{\infty}\gamma^t \nabla_\theta \log \pi_{\theta_k} \nabla_\theta \log \pi_{\theta_k}^\top \right] \right\| \leq G_1^2 \gamma^H:=\delta_F^2
    \end{equation}
    which implies $\delta_F = G_1^2 \gamma^H$.

    We bound the variance of the FIM estimator using the bias-variance decomposition:
    \begin{equation}
        \mathbb{E}\|\widehat{F}- F\|^2 = \mathbb{E}\|\widehat{F} - \mathbb{E}[\widehat{F}]\|^2 + \|\mathbb{E}[\widehat{F}] - F\|^2 \label{eq:bias-variance}
    \end{equation}
    The second term is bounded by $\delta_F^2 = G_1^4 \gamma^{2H}$. For the first term, since $\widehat{F}(\theta_k)$ is the sample mean of $B$ i.i.d. trajectory estimates $\widehat{F}^{(b)}$, the variance scales by $1/B$:
    \begin{equation}
        \mathbb{E}\|\widehat{F} - \mathbb{E}[\widehat{F}]\|^2 \leq \frac{1}{B} \mathbb{E}\|\widehat{F}^{(b)}\|^2 \leq \frac{G_1^4}{B}
    \end{equation}
    Substituting these back into \eqref{eq:bias-variance} yields:
    \begin{equation}
       \mathbb{E}\|\widehat{F} - F\|^2 \leq G_1^4 \left( \frac{1}{B} + \gamma^{2H} \right):=\sigma_F^2
    \end{equation}

    Finally, we bound the gradient of the objective function:
    \begin{equation}
        \|\nabla_\theta f(\mathbf{J}^{\pi_{\theta_k}})\| \leq \frac{1}{1-\gamma}\left\| \sum_{m=1}^{M} \frac{\partial f}{\partial J_m} A_m^{\pi_{\theta_k}}(s,a) \nabla_\theta \log \pi_{\theta_k}(a|s) \right\| \leq \frac{CMG_1}{(1-\gamma)^2}=:\Lambda_g
    \end{equation}

Finally, applying Jensen's inequality yields $\delta_g^2 \leq \sigma^2_g$. The result then follows by substituting these bounds into Theorem \ref{thm_2}.

\begin{align*}
\mathbb{E}\!\left[\|\omega_N-\omega^\ast\|^2\right]
\le \mathcal{O}\Bigg(
&\exp\!\left(-\frac{N\mu^2(1-\gamma)^2}{4CMG_1}\right) R_0^2
+ \frac{C^2M^2G_1^4}{\mu^2(1-\gamma)^4}\left(\frac{1}{B}+\gamma^{2H}\right)
+ \frac{C^2M^2G_1^4}{\mu^4(1-\gamma)^4}\gamma^H \\
&\qquad\qquad
+ \left(\sigma_g^2\!\left(\frac{1}{G_1^2}+\frac{1}{\mu^2}\right)\right)
\Bigg).
\end{align*}
and
\begin{align*}
\|\mathbb{E}[\omega_N]-\omega^\ast\|^2
\le \mathcal{O}\Bigg(
&\exp\!\left(-\frac{N\mu^2(1-\gamma)^2}{4CMG_1}\right) R_0^2
+ \frac{G_1^2\gamma^H}{\mu^2}\,R_0^2
+ \frac{G_1^6}{\mu^4}\,\gamma^H +\frac{\bar{\delta}_g^2}{\mu^2}
\Bigg).
\end{align*}
where  $R_0 = \|\omega_0 - \omega^\ast\|^2$, $\Lambda_g=  \frac{CMG_1}{(1-\gamma)^2}$, $\sigma^2_F=G_1^4 \left( \frac{1}{B} + \gamma^{2H} \right)$ and $\Lambda_F= G_1^2$.

\section{Proof of Lemma \ref{lem:grad-bias-variance}}
\label{app:grad-bias-variance}
Recall that the gradient estimator is defined as:
\begin{align}
g(\tau_i^H, \widehat{\mathbf{J}}_H^{\pi_\theta} | \theta) 
&:= \sum_{t=0}^{H-1} \nabla_\theta \log \pi_\theta(a_t^i | s_t^i) \left( \sum_{m=1}^{M} \partial_m f(\widehat{\mathbf{J}}_H^{\pi_\theta}) \sum_{h=t}^{H-1} \gamma^h r_m(s_h^i, a_h^i) \right).
\end{align}
The expectation of the estimator with respect to the truncated horizon objective is given by:
\begin{align}
\mathbb{E}[g(\tau_i^H, \mathbf{J}^{\pi_\theta}_H | \theta)] 
&= \mathbb{E} \left[ \sum_{t=0}^{H-1} \nabla_\theta \log \pi_\theta(a_t^i | s_t^i) \left( \sum_{m=1}^{M} \partial_m f(\mathbf{J}_H^{\pi_\theta}) \sum_{h=t}^{H-1} \gamma^h r_m(s_h^i, a_h^i) \right) \right].
\end{align}
Invoking Lemma~\ref{lem:J_prop}, we observe that the truncation error vanishes exponentially in $H$:
\begin{align}
\| \mathbb{E}[g(\tau_i^H, \mathbf{J}^{\pi_\theta}_H | \theta)] - \nabla_\theta f(\mathbf{J}^{\pi_\theta}) \| \leq \mathcal{O}(\gamma^H).
\end{align}

To bound the total bias, we decompose the error into terms reflecting the estimation error for the truncated objective and the horizon truncation error:
\begin{align}
    \mathbb{E}\left[g(\tau_i^H, \widehat{\mathbf{J}}^{\pi_\theta}_H | \theta) \right] - \nabla_\theta f(\mathbf{J}^{\pi_{\theta}}) 
    &= \mathbb{E}\left[g(\tau_i^H, \widehat{\mathbf{J}}^{\pi_\theta}_H | \theta) - g(\tau_i^H, \mathbf{J}^{\pi_\theta}_H | \theta) \right] \nonumber \\
    &\quad + \mathbb{E}\left[g(\tau_i^H, \mathbf{J}^{\pi_\theta}_H | \theta) \right] - \nabla_\theta f(\mathbf{J}^{\pi_{\theta}}).
\end{align}
For the first term on the right-hand side, we have:
\begin{align}
    &\left\| \mathbb{E}[g(\tau_i^H, \mathbf{J}^{\pi_\theta}_H | \theta)] - \mathbb{E}[g(\tau_i^H, \widehat{\mathbf{J}}^{\pi_\theta}_H | \theta)] \right\| \nonumber \\
    &= \left\| \mathbb{E} \left[ \sum_{t=0}^{H-1} \nabla_\theta \log \pi_\theta(a_t^i | s_t^i) \sum_{m=1}^{M} \left( \partial_m f(\mathbf{J}_H^{\pi_\theta}) - \partial_m f(\widehat{\mathbf{J}}_H^{\pi_\theta}) \right) \hat{Q}_{m,t} \right] \right\|,
\end{align}
where $\hat{Q}_{m,t} = \sum_{h=t}^{H-1} \gamma^h r_m(s_h^i, a_h^i)$. If the trajectories used for $\widehat{\mathbf{J}}_H^{\pi_\theta}$ are independent of $\tau_i^H$, the expectation factors, allowing us to bound the terms separately.
\begin{align}
&\|\E[g(\tau_i^H,\J^{\pi_\theta}_H | \theta)]-\E[g(\tau_i^H,\widehat{\J}^{\pi_\theta}_H | \theta)]\| \nonumber\\
&=
\norm{\E \left[\sum_{t=0}^{H-1}
\nabla_\theta \log \pi_\theta(a_t^i | s_t^i)\,
\Bigg(
\sum_{m=1}^{M}
\left(\pd[m] f\!\left(\J_H^{\pi_\theta}\right)-\pd[m] f\!\left(\widehat{\J}_H^{\pi_\theta}\right)\right)
\sum_{h=t}^{H-1} \gamma^h r_m(s_h^i,a_h^i)
\Bigg)\right]}\nonumber\\
&=
\norm{\sum_{t=0}^{H-1} \sum_{m=1}^{M} \E \left[
\nabla_\theta \log \pi_\theta(a_t^i | s_t^i) \left(\sum_{h=t}^{H-1} \gamma^h r_m(s_h^i,a_h^i)\right)
\left(\pd[m] f\!\left(\J_H^{\pi_\theta}\right)-\pd[m] f\!\left(\widehat{\J}_H^{\pi_\theta}\right)\right)
\right]}\nonumber\\
&=
\norm{\sum_{t=0}^{H-1} \sum_{m=1}^{M} \E \left[
\nabla_\theta \log \pi_\theta(a_t^i | s_t^i) \left(\sum_{h=t}^{H-1} \gamma^h r_m(s_h^i,a_h^i)\right)\right]\E
\left[\pd[m] f\!\left(\J_H^{\pi_\theta}\right)-\pd[m] f\!\left(\widehat{\J}_H^{\pi_\theta}\right)\right]
}\nonumber\\
&\leq
\sum_{t=0}^{H-1} \sum_{m=1}^{M} \norm{\E \left[
\nabla_\theta \log \pi_\theta(a_t^i | s_t^i) \left(\sum_{h=t}^{H-1} \gamma^h r_m(s_h^i,a_h^i)\right)\right]}\cdot \Big|
\pd[m] f\!\left(\J_H^{\pi_\theta}\right)-\E \left[\pd[m] f\!\left(\widehat{\J}_H^{\pi_\theta}\right)
\right]\Big|
\end{align}

Using boundedness of the score function and rewards, we obtain the following bound:
\begin{align}
    \norm{\E \left[
\nabla_\theta \log \pi_\theta(a_t^i | s_t^i) \left(\sum_{h=t}^{H-1} \gamma^h r_m(s_h^i,a_h^i)\right)\right]} &\leq 
\|\nabla_\theta \log \pi_\theta(a_t^i | s_t^i)\|\Bigg| \sum_{h=t}^{H-1} \gamma^h r_m(s_h^i,a_h^i)\Bigg| \nonumber\\
&\leq G_1 \sum_{h=t}^\infty \gamma^h = \frac{G_1\gamma^t }{1-\gamma }
\end{align}
Thus, we obtain
\begin{align}
\|\E[g(\tau_i^H,\J^{\pi_\theta}_H | \theta)]-\E[g(\tau_i^H,\widehat{\J}^{\pi_\theta}_H | \theta)]\|&\leq \sum_{t=1}^{H-1}\frac{G_1\gamma^t}{1-\gamma} \sum_{m=1}^{M}  \Big|
\pd[m] f\!\left(\J_H^{\pi_\theta}\right)-\E \left[\pd[m] f\!\left(\widehat{\J}_H^{\pi_\theta}\right)
\right]\Big|\nonumber\\
&\leq \frac{G_1}{(1-\gamma)^2} \sum_{m=1}^{M}  \Big|
\pd[m] f\!\left(\J_H^{\pi_\theta}\right)-\E \left[\pd[m] f\!\left(\widehat{\J}_H^{\pi_\theta}\right)
\right]\Big|
\end{align}

Similarly, we can bound the second order expected difference as follows
\begin{align}
&\E\|g(\tau_i^H,\J^{\pi_\theta}_H | \theta)-g(\tau_i^H,\widehat{\J}^{\pi_\theta}_H | \theta)\|^2\nonumber\\
&=
\E\norm{\sum_{t=0}^{H-1}
\nabla_\theta \log \pi_\theta(a_t^i | s_t^i)\,
\Bigg(
\sum_{m=1}^{M}
\left(\pd[m] f\!\left(\J_H^{\pi_\theta}\right)-\pd[m] f\!\left(\widehat{\J}_H^{\pi_\theta}\right)\right)
\sum_{h=t}^{H-1} \gamma^h r_m(s_h^i,a_h^i)
\Bigg)}^2\nonumber\\
&=
\E\norm{\sum_{t=0}^{H-1} \sum_{m=1}^{M} 
\nabla_\theta \log \pi_\theta(a_t^i | s_t^i) \left(\sum_{h=t}^{H-1} \gamma^h r_m(s_h^i,a_h^i)\right)
\left(\pd[m] f\!\left(\J_H^{\pi_\theta}\right)-\pd[m] f\!\left(\widehat{\J}_H^{\pi_\theta}\right)\right)
}^2
\end{align}
Note that
\begin{align}
&\norm{\sum_{t=0}^{H-1} \sum_{m=1}^{M} 
\nabla_\theta \log \pi_\theta(a_t^i | s_t^i) \left(\sum_{h=t}^{H-1} \gamma^h r_m(s_h^i,a_h^i)\right)
\left(\pd[m] f\!\left(\J_H^{\pi_\theta}\right)-\pd[m] f\!\left(\widehat{\J}_H^{\pi_\theta}\right)\right)
}\nonumber\\
&\leq \sum_{t=0}^{H-1} \sum_{m=1}^{M} 
\|\nabla_\theta \log \pi_\theta(a_t^i | s_t^i) \|\left|\sum_{h=t}^{H-1} \gamma^h r_m(s_h^i,a_h^i)\right|
\left|\pd[m] f\!\left(\J_H^{\pi_\theta}\right)-\pd[m] f\!\left(\widehat{\J}_H^{\pi_\theta}\right)\right|\nonumber\\
&\leq \sum_{t=0}^{H-1} \sum_{m=1}^{M} 
\frac{G_1\gamma^t}{1-\gamma}
\left|\pd[m] f\!\left(\J_H^{\pi_\theta}\right)-\pd[m] f\!\left(\widehat{\J}_H^{\pi_\theta}\right)\right|\nonumber\\
&\leq  \frac{G_1}{(1-\gamma)^2} \sum_{m=1}^{M} 
\left|\pd[m] f\!\left(\J_H^{\pi_\theta}\right)-\pd[m] f\!\left(\widehat{\J}_H^{\pi_\theta}\right)\right|
\end{align}
Thus, 
\begin{align}
\E\|g(\tau_i^H,\J^{\pi_\theta}_H | \theta)-g(\tau_i^H,\widehat{\J}^{\pi_\theta}_H | \theta)\|^2&\leq
\E \left( \frac{G_1}{(1-\gamma)^2} \sum_{m=1}^{M} 
\left|\pd[m] f\!\left(\J_H^{\pi_\theta}\right)-\pd[m] f\!\left(\widehat{\J}_H^{\pi_\theta}\right)\right|\right)^2\nonumber\\
&\leq
\frac{G_1^2M}{(1-\gamma)^4} \sum_{m=1}^{M}  \E 
\left|\pd[m] f\!\left(\J_H^{\pi_\theta}\right)-\pd[m] f\!\left(\widehat{\J}_H^{\pi_\theta}\right)\right|^2
\end{align}

\section{Proof of Lemma \ref{lem:mlmc} (iii)}
\label{app:mlmc}

For part $(iii)$ we compute a second-order moment bound: \begin{align}
\E \abs{\pd[m] f\left(\widehat{\mathbf{J}}_{H,\mathrm{MLMC}}^{\pi_\theta}\right)}^2 &= \E \abs{\pd[m] f\left(\widehat{\mathbf{J}}_{H,1}^{\pi_\theta}\right)
+
\mathbf{1}_{\{2^Q\le B_{\max}\}}\,2^{Q}
\Big(
\pd[m] f\left(\widehat{\mathbf{J}}_{H,2^{Q}}^{\pi_\theta}\right)
-
\pd[m] f\left(\widehat{\mathbf{J}}_{H,2^{Q-1}}^{\pi_\theta}\right)
\Big)}^2\nonumber\\
&\leq 2\E \abs{\pd[m] f\left(\widehat{\mathbf{J}}_{H,1}^{\pi_\theta}\right)}^2
+
2\E\abs{\mathbf{1}_{\{2^Q\le B_{\max}\}}\,2^{Q}
\Big(
\pd[m] f\left(\widehat{\mathbf{J}}_{H,2^{Q}}^{\pi_\theta}\right)
-
\pd[m] f\left(\widehat{\mathbf{J}}_{H,2^{Q-1}}^{\pi_\theta}\right)
\Big)}^2\nonumber
\end{align}
Note that
\begin{align}
   & \E\abs{\mathbf{1}_{\{2^Q\le B_{\max}\}}\,2^{Q}
\Big(
\pd[m] f\left(\widehat{\mathbf{J}}_{H,2^{Q}}^{\pi_\theta}\right)
-
\pd[m] f\left(\widehat{\mathbf{J}}_{H,2^{Q-1}}^{\pi_\theta}\right)
\Big)}^2 \nonumber\\
&=\sum_{q=1}^{\lfloor\log_2 B_{\max}\rfloor} \Pr(Q=q) 4^{q}
\E\Big(
\pd[m] f\left(\widehat{\mathbf{J}}_{H,2^{q}}^{\pi_\theta}\right)
-
\pd[m] f\left(\widehat{\mathbf{J}}_{H,2^{q-1}}^{\pi_\theta}\right)
\Big)^2\nonumber\\
&\leq \sum_{q=1}^{\lfloor\log_2 B_{\max}\rfloor}  2^{q}
L_f^2 \E \norm{\widehat{\mathbf{J}}_{H,2^{q}}^{\pi_\theta}
-
\widehat{\mathbf{J}}_{H,2^{q-1}}^{\pi_\theta}}^2\nonumber\\
&\leq \sum_{q=1}^{\lfloor\log_2 B_{\max}\rfloor}  2^{q+1}
L_f^2 \left(\E \norm{\widehat{\mathbf{J}}_{H,2^{q}}^{\pi_\theta}-\J_{H}^{\pi_\theta}}^2
+
\E\norm{\J_{H}^{\pi_\theta} - \widehat{\mathbf{J}}_{H,2^{q-1}}^{\pi_\theta}}^2\right)\nonumber\\
&\leq \sum_{q=1}^{\lfloor\log_2 B_{\max}\rfloor}  2^{q+1}
L_f^2 \left(\frac{M}{(1-\gamma)^22^q}
+
\frac{M}{(1-\gamma)^22^{q-1}}\right) \leq \frac{6 ML_f^2 \lfloor\log_2 B_{\max}\rfloor }{(1-\gamma)^2} 
\end{align}
This yields
\begin{align}
    \mathbb{E} \left\| \nabla_\theta f(\mathbf{J}^{\pi_{\theta}}) - g(\tau_i^H, \widehat{\mathbf{J}}^{\pi_\theta}_H | \theta) \right\|^2 
    &\leq \mathcal{O} \left( \frac{L_f^2G_1^2 M^3 \lfloor\log_2 B_{\max}\rfloor}{(1-\gamma)^6} \right),
\end{align}
and concludes the proof for Lemma \ref{lem:mlmc}.

\section{Proof of Lemma \ref{lem:vanilla}}
\label{app:vanilla}

Since $\partial_m f$ is $L_{2,f}$-smooth for each $m \in [M]$ by Assumption \ref{assump:second_smooth_f}, we have:
\begin{align}
    -\frac{L_{2,f}}{2}\|\widehat{\mathbf{J}}^{\pi_\theta}_{H,B} - \mathbf{J}^{\pi_\theta}_{H}\|^2 
    \leq \partial_m f(\widehat{\mathbf{J}}^{\pi_\theta}_{H,B}) - \partial_m f(\mathbf{J}^{\pi_\theta}_{H}) - \nabla \partial_m f(\mathbf{J}^{\pi_\theta}_{H})^\top (\widehat{\mathbf{J}}^{\pi_\theta}_{H,B} - \mathbf{J}^{\pi_\theta}_{H}) 
    \leq \frac{L_{2,f}}{2}\|\widehat{\mathbf{J}}^{\pi_\theta}_{H,B} - \mathbf{J}^{\pi_\theta}_{H}\|^2.\nonumber
\end{align}
Taking the expectation and utilizing the fact that the empirical return estimate is unbiased, i.e., $\mathbb{E}[\widehat{\mathbf{J}}^{\pi_\theta}_{H,B}] = \mathbf{J}^{\pi_\theta}_{H}$, the linear term $\mathbb{E}[\nabla \partial_m f(\mathbf{J}^{\pi_\theta}_{H})^\top (\widehat{\mathbf{J}}^{\pi_\theta}_{H,B} - \mathbf{J}^{\pi_\theta}_{H})]$ vanishes. This yields:
\begin{align}
    \left| \mathbb{E}[\partial_m f(\widehat{\mathbf{J}}^{\pi_\theta}_{H,B})] - \partial_m f(\mathbf{J}^{\pi_\theta}_{H}) \right| 
    \leq \frac{L_{2,f}}{2} \mathbb{E} \|\widehat{\mathbf{J}}^{\pi_\theta}_{H,B} - \mathbf{J}^{\pi_\theta}_{H}\|^2.
\end{align}
By Lemma \ref{lem:hat-J-variance}, we have $\mathbb{E} \|\widehat{\mathbf{J}}^{\pi_\theta}_{H,B} - \mathbf{J}^{\pi_\theta}_{H}\|^2 \leq \frac{M}{(1-\gamma)^2 B}$. Substituting this into the inequality above, we obtain:
\begin{equation}
    \left| \mathbb{E}[\partial_m f(\widehat{\mathbf{J}}^{\pi_\theta}_{H,B})] - \partial_m f(\mathbf{J}^{\pi_\theta}_{H}) \right| \leq \mathcal{O}\left( \frac{L_{2,f} M}{(1-\gamma)^2 B} \right).
\end{equation}
Applying this bound to Lemma \ref{lem:grad-bias-variance}$(i)$ concludes the proof of Lemma \ref{lem:vanilla}$(i)$. For part $(ii)$, observe that
\begin{align*}
    \E\left\| \nabla_\theta f(\mathbf{J}^{\pi_{\theta}}) - g(\tau_i^H, \widehat{\mathbf{J}}^{\pi_\theta}_{H,B} | \theta)\right\|^2 &\leq \mathcal{O} \left( \frac{G_1^2 M}{(1-\gamma)^4} \sum_{m=1}^{M} \mathbb{E} \left| \partial_m f(\mathbf{J}_H^{\pi_\theta}) - \partial_m f(\widehat{\mathbf{J}}_{H,B}^{\pi_\theta}) \right|^2 \right)\\
&\leq \mathcal{O} \left( \frac{L_f^2G_1^2 M^2}{(1-\gamma)^4}  \mathbb{E} \norm{ \mathbf{J}_H^{\pi_\theta} - \widehat{\mathbf{J}}_{H,B}^{\pi_\theta} }^2\right)\\
    &\leq \mathcal{O} \left( \frac{G_1^2 M^{3}L_f^2}{(1-\gamma)^6 B} \right),
\end{align*}
which yields Lemma \ref{lem:vanilla} $(ii)$.

\section{Proof of Theorems \ref{thm:mlmc} and \ref{thm:npg}}
\label{app:final}

\subsection{Overall Proof Framework}

We begin with the error decomposition given by Lemma \ref{lem:general-framework}.  
For the natural policy gradient update
\[
\theta_{k+1} = \theta_k + \alpha \omega_k,
\]
we have

\begin{align}
f(\J^{\pi^\ast}) - \frac{1}{K} \sum_{k=0}^{K-1} \mathbb{E}[f(\J^{\pi_{\theta_k}})]
\le& 
\frac{\sqrt{\varepsilon_{\mathrm{bias}}}}{1-\gamma}
+
\frac{G_1}{K} \sum_{k=0}^{K-1} \|\mathbb{E}[\omega_k] - \omega_k^\ast\|
+
\frac{G_2\alpha}{2K} \sum_{k=0}^{K-1} \mathbb{E}[\|\omega_k\|^2]
\nonumber \\
&+
\frac{1}{\alpha K}
\mathbb{E}_{s \sim d_\rho^{\pi^\ast}}
\left[
\mathrm{KL}(\pi^\ast(\cdot | s) \| \pi_{\theta_0}(\cdot | s))
\right].\label{eq:framework1}
\end{align}

To control $\mathbb{E}\|\omega_k\|^2$ we use

\[
\omega_k=\omega_k^\ast+(\omega_k-\omega_k^\ast)
\]

which implies

\begin{align}
\|\omega_k\|^2
\le
2\|\omega_k^\ast\|^2
+
2\|\omega_k-\omega_k^\ast\|^2 .
\end{align}

From Assumption \ref{assump:fnd},

\begin{align}
\|\omega_k^\ast\|^2
\le
\mu
\|
\nabla_\theta f(\J^{\pi_{\theta_k}})
\|^2 .
\end{align}

Thus

\begin{align}
\|\omega_k\|^2
\le
2\mu
\|
\nabla_\theta f(\J^{\pi_{\theta_k}})
\|^2
+
2
\|
\omega_k-\omega_k^\ast
\|^2 .
\label{eq:omega_bound}
\end{align}

Substituting \eqref{eq:omega_bound} into \eqref{eq:framework1} gives

\begin{align}
&f(\J^{\pi^\ast}) -
\frac{1}{K}\sum_{k=0}^{K-1}\mathbb{E}[f(\J^{\pi_{\theta_k}})]
\le
\frac{\sqrt{\varepsilon_{\mathrm{bias}}}}{1-\gamma}
+
\frac{G_1}{K}\sum_{k=0}^{K-1}\|\mathbb{E}[\omega_k]-\omega_k^\ast\|\nonumber\\
&\quad+
\frac{G_2\alpha\mu}{K}\sum_{k=0}^{K-1}
\mathbb{E}
\|
\nabla_\theta f(\J^{\pi_{\theta_k}})
\|^2
+
\frac{G_2\alpha}{K}\sum_{k=0}^{K-1}
\mathbb{E}
\|
\omega_k-\omega_k^\ast
\|^2
+
\frac{1}{\alpha K}
\mathbb{E}_{s\sim d_\rho^{\pi^\ast}}
\mathrm{KL}(\pi^\ast||\pi_{\theta_0}).
\label{eq:framework2}
\end{align}

From Lemma \ref{lem:stationary}, we have
\begin{align}
\Big(\tfrac{\alpha\mu}{2}-\alpha^2 L_J G^2\Big)\,\frac{1}{K}\sum_{k=0}^{K-1}\|\nabla_\theta f(\J^{\pi_{\theta_k}})\|^2
\le\;& \frac{\mathbb{E}\!\left[f(\J^{\pi_{\theta_{K}}}) - f(\J^{\pi_{\theta_{0}}})\right] }{K}\nonumber \\
&+ \frac{2\alpha}{\mu K}\sum_{k=0}^{K-1}\|\mathbb{E}[\omega_k]-\omega_k^\ast\|^2
+ \frac{L_J\alpha^2}{K}\sum_{k=0}^{K-1}\mathbb{E}\!\left[\|\omega_k-\omega_k^\ast\|^2\right].
\end{align}
Assuming $\alpha \le \frac{\mu}{4L_J G^2}$ (so that $\tfrac{\alpha\mu}{2}-\alpha^2 L_J G^2 \ge \tfrac{\alpha\mu}{4}$),
we can divide both sides to obtain
\begin{align}
\frac{1}{K}\sum_{k=0}^{K-1}\|\nabla_\theta f(\J^{\pi_{\theta_k}})\|^2
\le\;& \frac{4}{\alpha\mu K}\,\mathbb{E}\!\left[f(\J^{\pi_{\theta_{K}}}) - f(\J^{\pi_{\theta_{0}}})\right] \nonumber \\
&+ \frac{8}{\mu^2 K}\sum_{k=0}^{K-1}\|\mathbb{E}[\omega_k]-\omega_k^\ast\|^2
+ \frac{4L_J\alpha}{\mu K}\sum_{k=0}^{K-1}\mathbb{E}\!\left[\|\omega_k-\omega_k^\ast\|^2\right].\label{eq:gradbound}
\end{align}

Substituting \eqref{eq:gradbound} into \eqref{eq:framework2} yields the master bound

\begin{align}
f(\J^{\pi^\ast}) - \frac{1}{K} \sum_{k=0}^{K-1} \mathbb{E}[f(\J^{\pi_{\theta_k}})]
\le\;& \frac{\sqrt{\varepsilon_{\mathrm{bias}}}}{1-\gamma}
+ \frac{G_1}{K} \sum_{k=0}^{K-1} \|\mathbb{E}[\omega_k] - \omega_k^\ast\|
+ \frac{1}{\alpha K} \mathbb{E}_{s \sim d_\rho^{\pi^\ast}} \!\left[ \mathrm{KL}(\pi^\ast(\cdot | s) \| \pi_{\theta_0}(\cdot | s)) \right] \nonumber\\
&\; + 4G_2\,\frac{\mathbb{E}\!\left[f(\J^{\pi_{\theta_{K}}}) - f(\J^{\pi_{\theta_{0}}})\right]}{K}
+ \frac{8G_2\alpha}{\mu K}\sum_{k=0}^{K-1}\|\mathbb{E}[\omega_k]-\omega_k^\ast\|^2 \nonumber\\
&\; + \left(\frac{G_2\alpha}{K} + \frac{4G_2L_J\alpha^2}{K}\right)\sum_{k=0}^{K-1}\mathbb{E}\!\left[\|\omega_k-\omega_k^\ast\|^2\right].
\label{eq:master}
\end{align}

The remaining task is to bound the terms
$\|\mathbb{E}[\omega_k]-\omega_k^\ast\|$ and
$\mathbb{E}\|\omega_k-\omega_k^\ast\|^2$.

From Lemma \ref{lem:npg-error-bounds},

\begin{align}
\mathbb{E}\|\omega_k-\omega_k^\ast\|^2
\le
\mathcal{O}(T_1+T_2+T_3+T_4),
\end{align}

with

\begin{align}
T_1 &= 
\exp\!\left(
-\frac{N\mu^2(1-\gamma)^2}{4CMG_1}
\right)R_0^2,
&\qquad
T_2 &=
\frac{C^2M^2G_1^4}{\mu^2(1-\gamma)^4}
\left(\frac{1}{B}+\gamma^{2H}\right),
\\
T_3 &=
\frac{C^2M^2G_1^4}{\mu^4(1-\gamma)^4}\gamma^H,
&\qquad
T_4 &=
\sigma_g^2\!\left(\frac{1}{G_1^2}+\frac{1}{\mu^2}\right).
\end{align}

Similarly

\begin{align}
\|
\mathbb{E}[\omega_k]-\omega_k^\ast
\|^2
\le
\mathcal{O}(S_1+S_2+S_3+S_4),
\end{align}

where

\begin{align}
S_1 &= T_1,
&\qquad
S_2 &= \frac{G_1^2}{\mu^2}\gamma^H R_0^2,
&\qquad
S_3 &= \frac{G_1^6}{\mu^4}\gamma^H,
&\qquad
S_4 &= \frac{\bar{\delta}_g^2}{\mu^2}.
\end{align}

We now specialize this bound for the two estimators.

\subsection{Proof of Theorem \ref{thm:mlmc}}

We begin by specifying the parameters of Algorithm
\ref{alg:MLMC-NPG}.  
Set

\begin{align}
\alpha &= \frac{\mu\epsilon\log(1/\epsilon)}{4L_JG^2},&\qquad
B_{\max} &= \frac{1}{\epsilon^2}, &\qquad
K &= \Theta\!\left(\frac{1}{\alpha\epsilon}\right), &\qquad
B &= 1 .
\end{align}

The trajectory length and number of inner iterations are chosen as

\begin{align}
H &= 2\frac{\log(1/\epsilon)}{-\log\gamma},
&\qquad
N &=
\frac{4CMG_1}{\mu^2(1-\gamma)^2}
\log\!\left(\frac{R_0^2}{\epsilon^2}\right).
\end{align}

With this choice we obtain

\begin{align}
\gamma^H \le \epsilon^2,
\qquad
\gamma^{2H} \le \epsilon^4,
\label{eq:gamma_bounds}
\end{align}

and

\begin{align}
\exp\!\left(
-\frac{N\mu^2(1-\gamma)^2}{4CMG_1}
\right)R_0^2
\le \epsilon^2 .
\label{eq:T1bound}
\end{align}

\paragraph{Bounding $T_1,\dots,T_4$.}

From \eqref{eq:T1bound},

\begin{align}
T_1 \le \epsilon^2 .
\end{align}

Since $B=1$ and using \eqref{eq:gamma_bounds},

\begin{align}
T_2
&=
\frac{C^2M^2G_1^4}{\mu^2(1-\gamma)^4}
\left(
1+\gamma^{2H}
\right)
\nonumber\\
&\le
\frac{C^2M^2G_1^4}{\mu^2(1-\gamma)^4}
\left(
1+\epsilon^4
\right)
\nonumber\\
&=
\mathcal{O}\!\left(
\frac{1}{(1-\gamma)^4}
\right).
\end{align}

Similarly

\begin{align}
T_3
=
\frac{C^2M^2G_1^4}{\mu^4(1-\gamma)^4}\gamma^H
\le
\mathcal{O}\!\left(
\frac{\epsilon^2}{(1-\gamma)^4}
\right).
\end{align}

To bound $T_4$, we use Lemma \ref{lem:mlmc}.  
The MLMC gradient estimator satisfies

\begin{align}
\sigma_g^2
\le
\mathcal{O}\!\left(
\frac{L_f^2G_1^2M^3\log B_{\max}}{(1-\gamma)^6}
\right).
\end{align}

Since $B_{\max}=1/\epsilon^2$,

\begin{align}
\sigma_g^2
=
\mathcal{O}\!\left(
\frac{\log(1/\epsilon)}{(1-\gamma)^6}
\right).
\end{align}

Thus

\begin{align}
T_4
=
\mathcal{O}\!\left(
\frac{\log(1/\epsilon)}{(1-\gamma)^6}
\right).
\end{align}

Combining the above bounds we obtain

\begin{align}
\mathbb{E}\|\omega_k-\omega_k^\ast\|^2
=
\mathcal{O}\!\left(
\frac{\log(1/\epsilon)}{(1-\gamma)^6}
\right).
\label{eq:Tbound_final}
\end{align}

\paragraph{Bounding $S_1,\dots,S_4$.}

From \eqref{eq:T1bound},

\begin{align}
S_1 = T_1 \le \epsilon^2 .
\end{align}

Using \eqref{eq:gamma_bounds},

\begin{align}
S_2 = \mathcal{O}(\epsilon^2),
\qquad
S_3 = \mathcal{O}(\epsilon^2).
\end{align}

From Lemma \ref{lem:mlmc},

\begin{align}
\bar{\delta}_g
\le
\frac{G_1M^{3/2}}{(1-\gamma)^3\sqrt{B_{\max}}}.
\end{align}

With $B_{\max}=1/\epsilon^2$,

\begin{align}
\bar{\delta}_g
=
\mathcal{O}\!\left(
\frac{\epsilon}{(1-\gamma)^3}
\right).
\end{align}

Thus

\begin{align}
S_4
=
\mathcal{O}\!\left(
\frac{\epsilon^2}{(1-\gamma)^6}
\right).
\end{align}

Combining these bounds gives

\begin{align}
\|
\mathbb{E}[\omega_k]-\omega_k^\ast
\|
=
\mathcal{O}\!\left(
\frac{\epsilon}{(1-\gamma)^3}
\right).
\label{eq:Sbound_final}
\end{align}

\paragraph{Substituting the bounds into \eqref{eq:master}.}

Recall the  bound \eqref{eq:master}:

\begin{align}
f(\J^{\pi^\ast}) - \frac{1}{K} \sum_{k=0}^{K-1} \mathbb{E}[f(\J^{\pi_{\theta_k}})]
\le\;& \frac{\sqrt{\varepsilon_{\mathrm{bias}}}}{1-\gamma}
+ \frac{G_1}{K} \sum_{k=0}^{K-1} \|\mathbb{E}[\omega_k] - \omega_k^\ast\|
+ \frac{1}{\alpha K} \mathbb{E}_{s \sim d_\rho^{\pi^\ast}} \!\left[ \mathrm{KL}(\pi^\ast(\cdot | s) \| \pi_{\theta_0}(\cdot | s)) \right] \nonumber\\
&\; + 4G_2\,\frac{\mathbb{E}\!\left[f(\J^{\pi_{\theta_{K}}}) - f(\J^{\pi_{\theta_{0}}})\right]}{K}
+ \frac{8G_2\alpha}{\mu K}\sum_{k=0}^{K-1}\|\mathbb{E}[\omega_k]-\omega_k^\ast\|^2 \nonumber\\
&\; + \left(\frac{G_2\alpha}{K} + \frac{4G_2L_J\alpha^2}{K}\right)\sum_{k=0}^{K-1}\mathbb{E}\!\left[\|\omega_k-\omega_k^\ast\|^2\right].
\label{eq:master_sub}
\end{align}

From \eqref{eq:Sbound_final}, we have

\[
\|
\mathbb{E}[\omega_k]-\omega_k^\ast
\|
=
\mathcal{O}
\left(
\frac{\epsilon}{(1-\gamma)^3}
\right).
\]

Therefore

\begin{align}
\frac{G_1}{K}
\sum_{k=0}^{K-1}
\|
\mathbb{E}[\omega_k]-\omega_k^\ast
\|
&=
G_1
\cdot
\mathcal{O}
\left(
\frac{\epsilon}{(1-\gamma)^3}
\right)
\nonumber\\
&=
\mathcal{O}
\left(
\frac{\epsilon}{(1-\gamma)^3}
\right).
\label{eq:term1}
\end{align}

Next, using the same bound,

\begin{align}
\frac{8G_2\alpha}{\mu K}
\sum_{k=0}^{K-1}
\|
\mathbb{E}[\omega_k]-\omega_k^\ast
\|^2
&=
\frac{8G_2\alpha}{\mu}
\cdot
\mathcal{O}
\left(
\frac{\epsilon^2}{(1-\gamma)^6}
\right)
\nonumber\\
&=
\mathcal{O}
\left(
\frac{\alpha\epsilon^2}{(1-\gamma)^6}
\right)=
\widetilde{\mathcal{O}}
\left(
\frac{\epsilon^3}{(1-\gamma)^4}
\right).
\label{eq:term2}
\end{align}

From \eqref{eq:Tbound_final},

\[
\mathbb{E}
\|
\omega_k-\omega_k^\ast
\|^2
=
\mathcal{O}
\left(
\frac{\log(1/\epsilon)}{(1-\gamma)^6}
\right).
\]

Thus

\begin{align}
\left(
\frac{G_2\alpha
+
4G_2L_J\alpha^2}{K}
\right)
\sum_{k=0}^{K-1}
\mathbb{E}
\|
\omega_k-\omega_k^\ast
\|^2
\nonumber\\
=
\left(
G_2\alpha
+
4G_2L_J\alpha^2 
\right)
\mathcal{O}
\left(
\frac{\log(1/\epsilon)}{(1-\gamma)^6}\right) = \widetilde{\mathcal{O}}\left(
\frac{\epsilon}{(1-\gamma)^4}
\right).
\label{eq:term3}
\end{align}

Since $K=\Theta\!\left(\frac{1}{\alpha\epsilon}\right)$ and  the KL divergence term
$\mathrm{KL}(\pi^\ast(\cdot|s)\|\pi_{\theta_0}(\cdot|s))$
is bounded by a constant for bounded policy classes, we have
\begin{equation}
\frac{1}{\alpha K}
\mathbb{E}_{s\sim d_\rho^{\pi^\ast}}
\mathrm{KL}(\pi^\ast(\cdot|s)\|\pi_{\theta_0}(\cdot|s))
=
\mathcal{O}(\epsilon).
\end{equation}
The remaining  term,

\begin{equation}
\frac{4G_2}{K}
\mathbb{E}[f(\J^{\pi_{\theta_K}})-f(\J^{\pi_{\theta_0}})] =\mathcal{O}(1/K) = \widetilde{\mathcal{O}}(\epsilon^2)\label{lastt1}
\end{equation}

Combining \eqref{eq:term1}--\eqref{lastt1} with
\eqref{eq:master_sub} yields

\begin{align}
&f(\J^{\pi^\ast})
-
\frac{1}{K}
\sum_{k=0}^{K-1}
\mathbb{E}[f(\J^{\pi_{\theta_k}})]
\nonumber\\
&\le
\frac{\sqrt{\varepsilon_{\mathrm{bias}}}}{1-\gamma}
+
\mathcal{O}
\left(
\frac{\epsilon}{(1-\gamma)^3}
\right)
+
\widetilde{\mathcal{O}}
\left(
\frac{\epsilon}{(1-\gamma)^4}
\right) + \mathcal{O}(\epsilon)\\
&\le
\frac{\sqrt{\varepsilon_{\mathrm{bias}}}}{1-\gamma}
+
\widetilde{\mathcal{O}}
\left(
\frac{\epsilon}{(1-\gamma)^4}
\right).
\end{align}

\subsection{Proof of Theorem \ref{thm:npg}}

We begin by specifying the parameters of Algorithm
\ref{alg:NPG}.  
Set

\begin{align}
\alpha &= \frac{\mu}{4L_JG_1^2}, &\qquad
B &= \frac{1}{\epsilon(1-\gamma)^2}, &\qquad
K &= \Theta\!\left(\frac{1}{\alpha\epsilon}\right).
\end{align}

The trajectory length and number of inner iterations are chosen as

\begin{align}
H &= 2\frac{\log(1/\epsilon)}{-\log\gamma},
&\qquad
N &=
\frac{4CMG_1}{\mu^2(1-\gamma)^2}
\log\!\left(\frac{R_0^2}{\epsilon^2}\right).
\end{align}

With this choice we obtain

\begin{align}
\gamma^H \le \epsilon^2,
\qquad
\gamma^{2H} \le \epsilon^4,
\label{eq:gamma_bounds_vanilla}
\end{align}

and

\begin{align}
\exp\!\left(
-\frac{N\mu^2(1-\gamma)^2}{4CMG_1}
\right)R_0^2
\le \epsilon^2 .
\label{eq:T1bound_vanilla}
\end{align}

\paragraph{Bounding $T_1,\dots,T_4$.}

From \eqref{eq:T1bound_vanilla},

\begin{align}
T_1 \le \epsilon^2 .
\end{align}

Using \eqref{eq:gamma_bounds_vanilla},

\begin{align}
T_2
&=
\frac{C^2M^2G_1^4}{\mu^2(1-\gamma)^4}
\left(
\frac{1}{B}+\gamma^{2H}
\right)
\nonumber\\
&\le
\frac{C^2M^2G_1^4}{\mu^2(1-\gamma)^4}
\left(
(1-\gamma)^2\epsilon+\epsilon^4
\right)
\nonumber\\
&=
\mathcal{O}
\left(
\frac{\epsilon}{(1-\gamma)^2}
\right).
\end{align}

Similarly

\begin{align}
T_3
=
\frac{C^2M^2G_1^4}{\mu^4(1-\gamma)^4}\gamma^H
\le
\mathcal{O}
\left(
\frac{\epsilon^2}{(1-\gamma)^4}
\right).
\end{align}

To bound $T_4$, we use Lemma \ref{lem:vanilla}.  
The variance of the vanilla gradient estimator satisfies

\begin{align}
\sigma_g^2
\le
\mathcal{O}
\left(
\frac{L_f^2G_1^2M^3}{(1-\gamma)^6 B}
\right).
\end{align}

Since $B = 1/((1-\gamma)^2\epsilon)$, we obtain

\begin{align}
\sigma_g^2
=
\mathcal{O}
\left(
\frac{\epsilon}{(1-\gamma)^4}
\right).
\end{align}

Thus

\begin{align}
T_4
=
\mathcal{O}
\left(
\frac{\epsilon}{(1-\gamma)^4}
\right).
\end{align}

Combining the above bounds yields

\begin{align}
\mathbb{E}\|\omega_k-\omega_k^\ast\|^2
=
\mathcal{O}
\left(
\frac{\epsilon}{(1-\gamma)^4}
\right).
\label{eq:Tbound_vanilla}
\end{align}

\paragraph{Bounding $S_1,\dots,S_4$.}

From \eqref{eq:T1bound_vanilla},

\begin{align}
S_1 = T_1 \le \epsilon^2 .
\end{align}

Using \eqref{eq:gamma_bounds_vanilla},

\begin{align}
S_2 = \mathcal{O}(\epsilon^2),
\qquad
S_3 = \mathcal{O}(\epsilon^2).
\end{align}

From Lemma \ref{lem:vanilla},

\begin{align}
\bar{\delta}_g
\le
\frac{L_{2,f}G_1M^2}{(1-\gamma)^4 B}.
\end{align}

Substituting the value of $B$,

\begin{align}
\bar{\delta}_g
=
\mathcal{O}
\left(
\frac{\epsilon}{(1-\gamma)^2}
\right).
\end{align}

Thus

\begin{align}
S_4
=
\mathcal{O}
\left(
\frac{\epsilon^2}{(1-\gamma)^4}
\right).
\end{align}

Combining these bounds gives

\begin{align}
\|
\mathbb{E}[\omega_k]-\omega_k^\ast
\|
=
\mathcal{O}
\left(
\frac{\epsilon}{(1-\gamma)^2}
\right).
\label{eq:Sbound_vanilla}
\end{align}

\paragraph{Substituting the bounds into \eqref{eq:master}.}

Using \eqref{eq:Sbound_vanilla},

\begin{align}
\frac{G_1}{K}
\sum_{k=0}^{K-1}
\|
\mathbb{E}[\omega_k]-\omega_k^\ast
\|
=
\mathcal{O}
\left(
\frac{\epsilon}{(1-\gamma)^2}
\right).
\end{align}

Similarly,

\begin{align}
\frac{8G_2\alpha}{\mu K}
\sum_{k=0}^{K-1}
\|
\mathbb{E}[\omega_k]-\omega_k^\ast
\|^2
=
\mathcal{O}
\left(
\frac{\epsilon^2}{(1-\gamma)^2}
\right).
\end{align}

From \eqref{eq:Tbound_vanilla},

\begin{align}
\left(
\frac{G_2\alpha}{K}
+
\frac{4G_2L_J\alpha^2}{K}
\right)
\sum_{k=0}^{K-1}
\mathbb{E}
\|
\omega_k-\omega_k^\ast
\|^2
=
\mathcal{O}
\left(
\frac{\epsilon}{(1-\gamma)^2}
\right).
\end{align}

Further, 

\begin{equation}
\frac{1}{\alpha K}
\mathbb{E}_{s\sim d_\rho^{\pi^\ast}}
\mathrm{KL}(\pi^\ast(\cdot|s)\|\pi_{\theta_0}(\cdot|s))
=
\mathcal{O}(\epsilon).
\end{equation}
The remaining  term,

\begin{equation}
\frac{4G_2}{K}
\mathbb{E}[f(\J^{\pi_{\theta_K}})-f(\J^{\pi_{\theta_0}})] =\mathcal{O}(1/K) = \mathcal{O}(\epsilon).
\end{equation}

Combining the above bounds yields

\begin{align}
&f(\J^{\pi^\ast})
-
\frac{1}{K}
\sum_{k=0}^{K-1}
\mathbb{E}[f(\J^{\pi_{\theta_k}})]
\nonumber\\
&\le
\frac{\sqrt{\varepsilon_{\mathrm{bias}}}}{1-\gamma}
+
\mathcal{O}
\left(
\frac{\epsilon}{(1-\gamma)^2}
\right)
+
\mathcal{O}
\left(
\frac{\epsilon}{(1-\gamma)^2}
\right) + \mathcal{O}(\epsilon)\\
&\le
\frac{\sqrt{\varepsilon_{\mathrm{bias}}}}{1-\gamma}
+
{\mathcal O}
\left(
\frac{\epsilon}{(1-\gamma)^2}
\right).
\end{align}

\end{document}